\documentclass[11pt,twoside]{article} 

\usepackage{eqnarray,amsmath}
\usepackage[utf8]{inputenc} 
\usepackage[T1]{fontenc}    
\usepackage{booktabs}       
\usepackage{amsfonts}       
\usepackage{nicefrac}       
\usepackage{microtype}      
\usepackage{xcolor}         
\usepackage{subcaption}
\expandafter\def\csname 
ver@subfig.sty\endcsname{}
\usepackage{eqnarray,amsmath}
\usepackage{epsf}
\usepackage{epsfig}
\usepackage{fancyhdr}
\usepackage{graphics}
\usepackage{graphicx}
\usepackage{psfrag}
\usepackage{fullpage}
\usepackage{pdfpages}
\usepackage[font=small,labelfont=bf]{caption}

\usepackage{url}

\usepackage{color}

\usepackage{amsthm}
\usepackage{amsmath}
\usepackage{amssymb,bbm}
\usepackage[skip=0pt]{caption}  

\usepackage{textcomp}
\usepackage{siunitx}
\usepackage{wrapfig}
\usepackage{algorithm}

\usepackage{multirow}
\usepackage{multicol}
\usepackage{makecell}
\usepackage{colortbl} 
\usepackage{changepage}

\definecolor{customyellow}{HTML}{fedf8a}
\usepackage{eqnarray,amsmath}
\usepackage{epsf}
\usepackage{epsfig}
\usepackage{fancyhdr}
\usepackage{graphics}
\usepackage{graphicx}
\usepackage{psfrag}
\usepackage{fullpage}
\usepackage{pdfpages}
\usepackage{ragged2e}

\newtheorem{assumption}{Assumption}

\newtheorem{lemma}{Lemma}

\newtheorem{theorem}{Theorem}
\newtheorem{proposition}{Proposition}

\usepackage{eqnarray,amsmath}
\usepackage[utf8]{inputenc} 
\usepackage[T1]{fontenc}    
\usepackage{booktabs}       
\usepackage{amsfonts}       
\usepackage{nicefrac}       
\usepackage{microtype}      

\usepackage[colorlinks,linkcolor=magenta,citecolor=blue, pagebackref=true]{hyperref}
\renewcommand*{\backrefalt}[4]{%
    \ifcase #1 \footnotesize{(Not cited.)}%
    \or        \footnotesize{(Cited on page~#2.)}%
    \else      \footnotesize{(Cited on pages~#2.)}%
    \fi}
\usepackage[nameinlink,capitalize,noabbrev]{cleveref}

\usepackage{caption}
\usepackage{subcaption}

\usepackage{epsf}
\usepackage{epsfig}
\usepackage{fancyhdr}
\usepackage{graphics}
\usepackage{graphicx}
\usepackage{psfrag}

\usepackage{url}

\usepackage{color}

\usepackage{amsthm}
\usepackage{amsmath}
\usepackage{amssymb,bbm}
\usepackage[justification=centering]{caption}
\usepackage{algorithmic}
\usepackage{algorithm}
\usepackage{textcomp}
\usepackage{siunitx}
\usepackage{wrapfig}
\usepackage{algorithmic}
\usepackage{algorithm}
\usepackage{multirow}
\usepackage{multicol}

\def\1{\bm{1}}

\DeclareMathAlphabet{\mathsfit}{\encodingdefault}{\sfdefault}{m}{sl}
\SetMathAlphabet{\mathsfit}{bold}{\encodingdefault}{\sfdefault}{bx}{n}

\def\gA{{\mathcal{A}}}

\def\gD{{\mathcal{D}}}

\def\gG{{\mathcal{G}}}

\def\gI{{\mathcal{I}}}
\def\gJ{{\mathcal{J}}}

\def\gO{{\mathcal{O}}}

\DeclareMathOperator*{\argmax}{arg\,max}

\usepackage{eqnarray,amsmath}

\usepackage{amsthm}
\usepackage{amsmath}
\usepackage{amssymb,bbm}

\allowdisplaybreaks

\newcommand{\bbE}{\mathbb{E}}

\newcommand{\Gs}{G^*}
\newcommand{\ks}{k^*}
\newcommand{\gasi}{\gamma^*_{i}}
\newcommand{\alsi}{\alpha^*_{i}}
\newcommand{\besi}{\beta^*_{i}}
\newcommand{\tas}{\tau^*}

\newcommand{\asis}{a^*_{is}}
\newcommand{\asil}{a^*_{i\ell}}

\newcommand{\bsis}{b^*_{is}}
\newcommand{\bsil}{b^*_{il}}

\newcommand{\gasj}{\gamma^*_{j}}
\newcommand{\alsj}{\alpha^*_{j}}
\newcommand{\besj}{\beta^*_{j}}

\begin{document}

\begin{center}

{\bf{\LARGE{Rethinking Multinomial Logistic Mixture of Experts with \\ \vspace{0.1em} Sigmoid Gating Function}}}
  
\vspace*{.2in}
{\large{
\begin{tabular}{ccc}
    Tuan Minh Pham$^{\star,\diamondsuit}$ & Thinh Cao$^{\star,\diamond}$ & Viet Nguyen$^{\star,\diamond}$ \\
    Huy Nguyen$^{\dagger}$ & Nhat Ho$^{\star\star,\dagger}$ & Alessandro Rinaldo$^{\star\star,\dagger}$
\end{tabular}
}}

\vspace*{.2in}

\begin{tabular}{c}
$^{\dagger}$The University of Texas at Austin\\
$^{\diamondsuit}$Purdue University\\
$^{\diamond}$Hanoi University of Science and Technology
\end{tabular}

\vspace*{.2in}
\today


\begin{abstract}
    The sigmoid gate in mixture-of-experts (MoE) models has been empirically shown to outperform the softmax gate across several tasks, ranging from approximating feed-forward networks to language modeling. Additionally, recent efforts have demonstrated that the sigmoid gate is provably more sample-efficient than its softmax counterpart under regression settings. Nevertheless, there are three notable concerns that have not been addressed in the literature, namely (i) the benefits of the sigmoid gate have not been established under classification settings; (ii) existing sigmoid-gated MoE models may not converge to their ground-truth; and (iii) the effects of a temperature parameter in the sigmoid gate remain theoretically underexplored. To tackle these open problems, we perform a comprehensive analysis of multinomial logistic MoE equipped with a modified sigmoid gate to ensure model convergence. Our results indicate that the sigmoid gate exhibits a lower sample complexity than the softmax gate for both parameter and expert estimation. Furthermore, we find that incorporating a temperature into the sigmoid gate leads to a sample complexity of exponential order due to an intrinsic interaction between the temperature and gating parameters. To overcome this issue, we propose replacing the vanilla inner product score in the gating function with a Euclidean score that effectively removes that interaction, thereby substantially improving the sample complexity to a polynomial order. 
\end{abstract}
\end{center}
\let\thefootnote\relax\footnotetext{$^{\star}$Equal contribution, $^{\star\star}$Co-last authors.}

\section{Introduction}
\label{sec:introduction}

First introduced by \cite{Jacob_Jordan-1991}, Mixture of Experts (MoE) is an ensemble framework that integrates multiple specialized sub-models through an adaptive gating mechanism. These sub-models, referred to as experts, can take various forms, including classifiers \cite{chen2022theory, yuksel2012twenty}, regression models \cite{kwon_em_2020,faria2010regression}, or feed-forward neural networks (FFNs) \cite{lepikhin_gshard_2021,jiang2024mixtral}. The gating mechanism is responsible for determining the contribution or weight of each expert conditioned on the input value, allowing experts to focus on different regions of the input space or aspects of the task. More recently, \cite{shazeer2017topk} proposed a sparse version of MoE architectures in which each input is routed to only a few relevant experts rather than all of them. Since this sparse expert activation provides a practical path toward scaling model capacity without a proportional increase in computation, MoE has gained renewed attention in large-scale applications, namely, large language models \cite{comanici2025gemini25pushingfrontier,grattafiori2024llama3,abdin2024phi3,qwen2025}, computer vision \cite{liang_m3vit_2022,Riquelme2021scalingvision}, multimodal learning \cite{han2024fusemoe,yun2024flexmoe}, domain generalization \cite{nguyen2025cosine,li2023sparse}, transfer learning \cite{diep2025on,le2025revisiting,truong2025replora}, among others.

\vspace{0.5em}
\noindent
Despite their success, sparse MoE models were shown by \cite{chi_representation_2022} to pose a new challenge known as representation collapse, where some experts dominated the computation while others remained underutilized. This issue is attributed to the expert competition induced by the widely used softmax gate, that is, when the weight of one expert increases, those of others automatically decrease. To this end, \cite{chi_representation_2022} opted for using a sigmoid gate with a learnable temperature parameter as a robust alternative to representation collapse. Next, \cite{csordas2023approximating} empirically demonstrated that the sigmoid gate yielded better performance than the softmax gate in the task of approximating two-layer FFNs. This finding aligns with the results in the technical reports of the DeepSeek language models, which indicate that DeepSeek-v3 \cite{deepseekv3} with sigmoid gating MoE outperforms DeepSeek-v2 \cite{deepseekv2} with softmax gating MoE. 

\vspace{0.5em}
\noindent
From a theoretical perspective, there have been a few efforts in the literature to elucidate the advantages of the sigmoid gate over the softmax gate.  \cite{nguyen2024sigmoid} evaluated the sample complexity of a regression framework in which the regression function takes the form of a sigmoid gating MoE. They proved that when employing the sigmoid gate, one needs significantly less data to approximate experts with a given error than when using the softmax gate.  
However, there are still several important points that have not been addressed. Firstly, the aforementioned benefits of the sigmoid gate have been validated for regression tasks only, while those for classification tasks have not been justified yet. Secondly, when the number of experts is unknown and over-specified (which is often the case), sigmoid gating MoE models considered in these works cannot converge to their ground-truth models due to the inherent sigmoid structure (see Corollary 1 in \cite{nguyen2024sigmoid}).
Lastly, the effect of the temperature parameter suggested by \cite{chi_representation_2022} has not been investigated.

\vspace{0.5em}
\noindent
\textbf{Contributions.} In this paper, we tackle the aforementioned problems and make the following contributions.

\vspace{0.5em}
\noindent
\emph{Demystify sigmoid gate in MoE for classification.} We study the theoretical properties of the \emph{sigmoid-gated multinomial logistic MoE (MLMoE)} model. We derive convergence rates for parameter and expert estimation and demonstrate that the sigmoid gate is more sample-efficient than the softmax gate studied in \cite{nguyen2024general}. 

\vspace{0.5em}
\noindent
\emph{Modified sigmoid gate for MLMoE convergence.} To ensure convergence of the MLMoE models to their ground-truth in the overspecified setting when the number of deployed experts is larger than the true value, we devise a modified sigmoid gate in which the sigmoid function is multiplied by a positive constant. We provide convergence rates of the MLMoE model and its parameters in Section~\ref{sec:without_temperature}.

\vspace{0.5em}
\noindent
\emph{Sigmoid gate with temperature.} In Section~\ref{sec:with_temperature}, we investigate the effects of the temperature on the model's sample complexity. Our results indicate that when using the standard inner product affinity score, one will need exponentially many data points to approximate the parameters or the experts within a given accuracy. This is primarily due to an intrinsic interaction between the temperature and gating parameters. To alleviate this problem, we suggest using instead the Euclidean affinity score, where the sample complexity is provably reduced to a polynomial order.

\vspace{0.5em}
\noindent
In Section~\ref{sec:experiments}, we describe several numerical experiments to verify empirically our theoretical results. Finally, in Section~\ref{sec:discussion}, we offer a thorough discussion of the paper, including practical implications that can be drawn from the theories, some limitations of our work, and a few potential directions. Proofs and additional results are in the appendices.
\vspace{0.5em}
\noindent

\begin{table*}[t]
\centering
\captionsetup{justification=justified,singlelinecheck=false}
\caption{Summary of parameter estimation rates for multinomial logistic mixture-of-experts (MLMoE) models under Voronoi-type metrics. For the softmax-gated MLMoE, we distinguish the two regimes of expert intercept parameters $b_{j\ell}^\ast$ (see \cite{nguyen2024general}): \emph{Regime 1} requires $\forall j\in[k^\ast],\ \exists \ell\in[K-1]\text{ such that } b_{j\ell}^\ast\neq 0$, while \emph{Regime 2} assumes $\exists j\in[k^\ast]\text{ such that } b_{j\ell}^\ast=0\ \forall \ell\in[K-1]$. For sigmoid gating, we show different gating mechanisms: no temperature, temperature with inner product affinity, and temperature with Euclidean affinity. Rates are reported for exact-specified and over-specified parameters. The Voronoi metrics are $\mathcal{D}_2$ in \cite{nguyen2024general} for softmax, and $\mathcal{D}_1$, $\mathcal{D}_{2,2}$, and $\mathcal{D}_3$ for the respective sigmoid settings.}
\label{tab:convergence-rate}
{%
\resizebox{\textwidth}{!}{
\begin{tabular}{|l|l|c|c|l|}
\hline
{\textbf{Gating functions}} & \textbf{Regime} & \textbf{Exact-specified experts} & \textbf{Over-specified experts} & \textbf{Theorems} \\
\hline
\multirow{2}{*}{Softmax} & Regime 1 & $\mathcal{O}_P([\log(n)/n]^{\frac{1}{2}})$ & $\mathcal{O}_P([\log(n)/n]^{\frac{1}{4}})$ &  \cite{nguyen2024general} \\
\cline{2-5}
& Regime 2 & \multicolumn{2}{c|}{
slower than any polynomial rate} &  \cite{nguyen2024general} \\
\hline
Sigmoid, No Temperature & \multirow{3}{*}{All} & $\mathcal{O}_P([\log(n)/n]^{\frac{1}{2}})$ & $\mathcal{O}_P([\log(n)/n]^{\frac{1}{4}})$ & Theorem~\ref{theorem:without_temperature} \\
\cline{1-1} \cline{3-5}
Sigmoid, Temp \& Inner Product &  & \multicolumn{2}{c|}{slower than any polynomial rate} & Theorem~\ref{theorem:with_temperature_inner} \\
\cline{1-1} \cline{3-5}
Sigmoid, Temp \& Euclidean &  & $\mathcal{O}_P([\log(n)/n]^{\frac{1}{2}})$ & $\mathcal{O}_P([\log(n)/n]^{\frac{1}{4}})$ & Theorem~\ref{theorem:with_temperature_euclidean} \\
\hline
\end{tabular}%
}
}
\end{table*}

\vspace{0.5em}
\noindent
\textbf{Notation.} For any $n\in\mathbb{N}$, we set $[n] := \{1,2,\ldots,n\}$. For any set $S$, $|S|$ denotes its cardinality. For a vector $v = (v_1,\ldots,v_d) \in \mathbb{R}^d$, $\|v\|$ is its Euclidean norm. For a multi-index $\alpha = (\alpha_1,\ldots,\alpha_d) \in \mathbb{N}^d$, we write $|\alpha| = \sum_{i=1}^d \alpha_i$, $\alpha! = \alpha_1! \cdots \alpha_d!$, and for $v \in \mathbb{R}^d$, $v^{\alpha} = v_1^{\alpha_1} \cdots v_d^{\alpha_d}$.
For positive sequences $(a_n)$ and $(b_n)$, we write $a_n = \mathcal{O}(b_n)$ or $a_n \lesssim b_n$ if $a_n \le C b_n$ for all $n$ and some universal constant $C>0$. For random positive sequences $(t_n)$ and $(s_n)$, we write $t_n = \mathcal{O}_P(s_n)$ if $t_n/s_n$ is stochastically bounded: for any $\epsilon>0$ there exists $M>0$ such that $\mathbb{P}(t_n/s_n > M) < \epsilon$ for all sufficiently large $n$. Finally, for any two probability densities $p$ and $q$ dominated by the Lebesgue measure $\mu$, we define the Total Variation distance as $d_V(p,q):= \frac{1}{2}\int|p-q|\mathrm{d}\mu$ and the Hellinger distance as $d_H(p,q):= \left(\frac{1}{2}\int|\sqrt{p}-\sqrt{q}|^2\mathrm{d}\mu\right)^{1/2}$. 



\section{Related Work}
A growing body of literature has investigated MoE under different likelihoods and gating mechanisms. Early theoretical work established consistency for specialized formulations: \cite{chen_improved_1999} developed improved learning algorithms for multiclass MoE by identifying instabilities in the EM inner loop and proposing a Newton-Raphson approach with the exact Hessian matrix, while \cite{pham2022functional} developed a functional MoE framework for classification that utilizes sparse regularization to identify key features in data. More recently, \cite{nguyen2024general} developed a general convergence theory for softmax-gated multinomial logistic MoE models, showing that intrinsic interactions between gating and expert parameters fundamentally determine estimation rates. 
Complementing these analyses, \cite{bing2025learninglargesoftmaxmixtures} provided a comprehensive theoretical analysis of the EM algorithm for learning large-scale softmax mixtures.

\vspace{0.5em}
\noindent
From a deep learning perspective, MoE architectures have been extensively studied for their expressive power and scalability. \cite{chen2022theory} provided a theoretical analysis of MoE layers in neural networks, characterizing their approximation capabilities and gradient-based training dynamics. \cite{wang2025expressive} analyzed the expressive capacity of MoE models for structured tasks, showing that expert composition enables efficient representation of hierarchical patterns. \cite{rigollet2025granularity} proved that finer expert granularity substantially enhances representational efficiency through improved specialization. In large language models, DeepSeek-MoE \cite{dai2024deepseekmoe,nguyen2025deepseekmoe} demonstrated empirically that shared expert isolation improves both performance and computational efficiency through sparse activation patterns.

\vspace{0.5em}
\noindent
While very insightful and valuable, these studies provide limited insights into the statistical challenges of MoE models and several theoretical gaps remain. In particular, sigmoid-gated MoE models have not been systematically studied under multinomial logistic classification, nor has convergence of the mixing measure been established in over-specified regimes. Moreover, the statistical role of temperature scaling under sigmoid gating is still poorly understood. We will address these gaps in subsequent sections.

\section{Modified Sigmoid Gating Function}
\label{sec:without_temperature}

In this section, we analyze the convergence of both the conditional density and the model parameters in the MLMoE model with a modified sigmoid gating function.

\vspace{0.5em}
\noindent
\textbf{Problem setup.}
Suppose that data $\{(X_i,Y_i)\}_{i=1}^{n}\subset\mathbb{R}^d\times [K]$
are sampled in an i.i.d. manner from a {\it sigmoid gating MLMoE model,} whose conditional density function is given, for any $s \in [K]$, by
\begin{align}\label{eq:standard_density}
p_{G^\ast}(Y=s \mid X)
:=& \sum_{i=1}^{k^\ast}
\frac{\exp(\gamma_i^\ast)\,
\sigma\big((\alpha_i^\ast)^\top X + \beta_i^\ast\big)}
{\sum_{j=1}^{k^\ast}\exp(\gamma_j^\ast)\,
\sigma\big((\alpha_j^\ast)^\top X + \beta_j^\ast\big)}\cdot f(Y=s | X; a_i^\ast, b_i^\ast),
\end{align}
where $t \in \mathbb{R} \mapsto \sigma(t) := \frac{1}{1+\exp(-t)}$ is the sigmoid function and $k^*$ denotes the number of ground-truth experts. Above, we incorporate the term $\exp(\gamma^*_i)$ into the standard sigmoid gate and refer to it as a modified sigmoid gating function, which we will elaborate later in this section.
Next, for each $i \in [\ks]$, the $i$-th expert takes the form of a multinomial logistic function  
\[
f(Y=s | X; a_i^\ast, b_i^\ast)
:= \frac{\exp\big((a_{is}^\ast)^\top X + b_{is}^\ast\big)}
{\sum_{\ell=1}^{K}\exp\big((a_{i\ell}^\ast)^\top X + b_{i\ell}^\ast\big)},
\]
where $a_i^*:= (a_{i1}^*, \ldots, a_{iK}^*) \in \mathbb{R}^{d\times K}$ and $b_i^* := (b_{i1}^*, \ldots, b_{iK}^*) \in \mathbb{R}^{K}$ are expert parameters. In contrast, we refer to $(\gamma_i^*,\alpha_i^*,\beta_i^* )_{i \in [k^*]} \subset \mathbb{R} \times \mathbb{R}^d \times  \mathbb{R}$ as gating parameters. Finally, $G^*:= \sum_{i=1}^{k^\ast}
\exp(\gamma_i^\ast)\,
\delta_{(\alpha_i^\ast,\beta_i^\ast,a_i^\ast,b_i^\ast)}$ is the true but unknown \emph{mixing measure,} consisting of a mixture of Dirac measures $\delta$ associated with the expert and gating parameters $(\gamma_i^\ast,\alpha_i^\ast,\beta_i^\ast,a_i^\ast,b_i^\ast) \in\Theta\subset \mathbb{R}\times \mathbb{R}^d \times \mathbb{R} \times \mathbb{R}^{d\times K}\times \mathbb{R}^K$, for all
$i\in[k^*]$, where $\Theta$ stands for the parameter space.

\vspace{0.5em}
\noindent
\textbf{Maximum likelihood estimation (MLE).} We estimate model parameters via estimating true mixing measure $G^*$ using the maximum likelihood method. However, since the number of true experts $\ks$ is typically unknown in practice, we will consider an \emph{over-specified} setting where we fit the model~\eqref{eq:standard_density} with a mixture of $k$ experts, where $k>\ks$. Then we consider the MLE
\begin{align}
    \label{eq:MLE}
    \widehat{G}_n\in\argmax_{G\in\mathcal{G}_k(\Theta)}\sum_{i=1}^{n}\log(p_{G}(Y_i|X_i)),
\end{align}
where $\mathcal{G}_k(\Theta)$ denotes the set of all mixing measures with at most $k$ components of the form $G=\sum_{i=1}^{k'}\exp(\gamma_{i})\delta_{(\alpha_{i},\beta_{i},a_{i},b_{i})}$, in which $1\leq k'\leq k$ and $(\gamma_{i},\alpha_{i},\beta_{i},a_{i},b_{i})\in\Theta$.

\vspace{0.5em}
\noindent
\textbf{Benefits of the modified sigmoid gate.}
The inclusion of the term $\exp(\gamma_i)$
is crucial for accurately characterizing the convergence of conditional density estimation. Specifically, since we fit the ground-truth MLMoE model \eqref{eq:standard_density} with a mixture of $k > k^*$ experts, some ground-truth parameters must be fitted by at least two components. For instance, suppose that $(\widehat{\alpha}_i^n, \widehat{\beta}_i^n, \widehat{a}_i^n, \widehat{b}_i^n) \to (\alpha_1^*, \beta_1^*, a_1^*, b_1^*)$ in probability, for $i \in \{1,2\}$. In order that the convergence of $p_{\widehat{G}_n}$ to $p_{G^*}$ holds, we must have
$
\sum_{i=1}^2 \exp(\widehat{\gamma}_i^n)\sigma\left((\widehat{\alpha}_i^n)^\top X+\widehat{\beta}_i^n\right) \to \exp(\gamma_i^*)\sigma\left((\alpha_i^*)^\top X+\beta_i^*\right)
$, for almost surely $X$. In the absence of $\exp(\gamma_i)$, the previous convergence no longer holds. Therefore, the inclusion of positive terms $\exp(\gamma_i)$ plays a vital role in the conditional density convergence. 

\vspace{0.5em}
\noindent
Before presenting the main results of this section, let us introduce some essential yet mild assumptions used throughout the paper unless stated otherwise.

\vspace{0.5em}
\noindent
\textbf{Assumptions.} In our analysis, we make four main assumptions on the model parameters.

\noindent
(\textbf{A.1}) The parameter space $\Theta$ is a compact subset of $\mathbb{R}\times \mathbb{R}^d \times \mathbb{R} \times \mathbb{R}^{d\times K}\times \mathbb{R}^K$ and the input space $\mathcal{X}$ is bounded;

\noindent
(\textbf{A.2}) $(a_1^*, b_1^*), \ldots, (a_{k^*}^*, b_{k^*}^*)$ are distinct parameters;

\noindent
(\textbf{A.3}) $a_{iK}^* = \mathbf{0}_d$, $b_{iK}^* = 0$, for all $i \in [k^*]$, and $\gamma_{k^*} = 0$; 

\noindent
(\textbf{A.4}) At least one among $\alpha_1^*, \ldots, \alpha_{k^*}^*$ is non-zero.

\vspace{0.5em}
\noindent
The first assumption is necessary for the MLE convergence. Next, the second enforces distinct experts, which together with the third assumption, enables the identifiable property of the sigmoid-gated MLMoE. Finally, the last assumption ensures that the gating values are input-dependent. 

\vspace{0.5em}
\noindent
Now, we begin our analysis. Firstly, we show that the sigmoid-gated MLMoE model is identifiable in
the following proposition, whose proof is in Appendix~\ref{appendix:Identifiability}.

\begin{proposition}
    \label{prop:Identifiability} (Identifiability) Assume that $G$ and $G'$ are two mixing measures in $\mathcal{G}_k(\Theta)$. If $p_G(Y|X) = p_{G'}(Y|X)$ with probability one (with respect to the joint distribution of $X$ and $Y$), then $G \equiv G'$. 
\end{proposition}

\noindent
This result indicates that convergence of the MLE $\widehat{G}_n$ to the true mixing measure 
$G^{\ast}$ is implied by the convergence of the density estimator $p_{\widehat{G}_n}$ to the true density 
$p_{G^{\ast}}$. Next, we derive the rate of the convergence for the latter. 
\begin{proposition} (Density estimation rate)
    \label{prop:density_rate}
    The density estimator
    $p_{\widehat{G}_n}(Y|X)$ converges to the true density
    $p_{\Gs}(Y|X)$ under the Hellinger distance at the rate
    \begin{align*}
        \bbE_X[d_{H}(p_{\widehat{G}_n}(\cdot|X),p_{\Gs}(\cdot|X))]
        =\mathcal{O}_P(\sqrt{Kdk\log(n)/n}).
    \end{align*}
\end{proposition}

\noindent
The proof of Proposition~\ref{prop:density_rate} is in
Appendix~\ref{appendix:density_rate}. The above bound implies that the rate for estimating the ground-truth density is parametric with respect to the sample size, while keeping $K$, $k$ and $d$ fixed.
Thus, if we can construct a metric $\gD$ among parameters that satisfies the inequality $\bbE_X[d_{H}(p_{\widehat{G}_n}(\cdot|X),p_{\Gs}(\cdot|X))] \gtrsim \gD(\widehat{G}_n,G^*)$, then the MLE $\widehat{G}_n$ will also converge to the true mixing measure $G^*$ at the parametric rate. To this end, we adopt the concept of Voronoi loss function introduced in \cite{manole22refined}.

\vspace{0.5em}
\noindent
\textbf{Voronoi loss.} Assume that a mixing measure $G$ has $k'$ atoms $\theta_i := (\alpha_i, \beta_i, a_i, b_i)$. We assign these atoms to the Voronoi cells $\gA_j := \gA_j(G)$ induced by the components $\theta_j^* := (\alpha_i^*, \beta_i^*, a_i^*, b_i^*)$ of $G^*$, which are defined as
\begin{equation}
    \gA_j := \{i \in [k']: \|\theta_i - \theta_j^*\| \le \|\theta_i - \theta_\ell^*\|, \forall \, \ell \ne j\}.
\end{equation}
Next, let us define $R_{ij}(\kappa_1,\kappa_2,\kappa_3,\kappa_4) := \|\Delta\alpha_{ij}\|^{\kappa_1}
+ |\Delta\beta_{ij}|^{\kappa_2}
+ \sum_{\ell=1}^K
\big(\|\Delta a_{ij\ell}\|^{\kappa_3}
+ |\Delta b_{ij\ell}|^{\kappa_4}\big) $,
where $\{\kappa_l, l=1,\ldots,4 \}$ are non-negative integers, $\Delta \alpha_{ij} := \alpha_i - \alpha_j^*$,
$\Delta \beta_{ij} := \beta_i - \beta_j^*$,
$\Delta a_{ij\ell} := a_{i\ell} - a_{j\ell}^*$,
and $\Delta b_{ij\ell} := b_{i\ell} - b_{j\ell}^*$. We consider the Voronoi loss given by
\begin{align}\label{eq:voronoi_loss_1}
    \gD_1(G, \Gs)
    &:= \sum_{j=1}^{\ks}
    \left|
    \sum_{i\in\gA_j} \exp(\gamma_i)
    - \exp(\gasi)
    \right| \nonumber\\
    &\quad
    + \sum_{\substack{j: |\gA_j| >1, \\ i\in\gA_j}}
    \exp(\gamma_i)\,
    R_{ij}(2,2,2,2)
    + \sum_{\substack{j: |\gA_j| =1, \\ i\in\gA_j}}
    \exp(\gamma_i)\,
    R_{ij}(1,1,1,1).
\end{align}
The above Voronoi loss function enables us to characterize the different convergence behaviors of parameter estimation and expert estimation in the sigmoid-gated MLMoE model, as a consequence of the following lower bound.
\begin{theorem}
\label{theorem:without_temperature}
For all mixing measures $G \in \mathcal{G}_k(\Theta)$, it holds that 
\[
\mathbb{E}_X \big[ d_H\big(p_G(\cdot \mid X),\, p_{G^\ast}(\cdot \mid X)\big) \big]
\;\gtrsim\;
\mathcal{D}_1(G, G^\ast).
\]
As a result, combining this inequality with the density estimation rate in
Proposition~\ref{prop:density_rate} yields
\[
\mathcal{D}_1(\widehat{G}_n, G^\ast)
= \mathcal{O}_P\!\left(\sqrt{K d k \log (n)/n}\right).
\]
\end{theorem}

\noindent
The proof of Theorem~\ref{theorem:without_temperature} is presented in
Appendix~\ref{appendix:without_temperature}. From the construction of $\gD_1$, it follows that {\it exactly-specified} parameters $(\alpha_j^*, \beta_j^*, a_j^*, b_j^*)$ (in the sense that their Voronoi cells $\gA_j$ contain a single element) can be estimated at the parametric rate $\gO_P([\log(n)/n]^{\frac{1}{2}})$, keeping all the other parameters fixed. We recall that each expert $f(\cdot | X; a_j^*, b_j^*)$ is a Lipschitz function. Let $\widehat{G}_n := \sum_{i=1}^{\widehat{k}_n} \exp(\widehat{\gamma}_i^n)\delta_{(\widehat{\alpha}_i^n, \widehat{\beta}_i^n, \widehat{a}_i^n, \widehat{b}_i^n)}$, then we obtain 
\begin{align}\label{eq:expert_rate}
  &\sup_{x} |f(\cdot|x;\widehat{a}_i^n, \widehat{b}_i^n) -f(\cdot|x;a_j^*, b_j^*) |\le L_1 \left(\|\widehat{a}_i^n - a_j^*\|  + \|\widehat{b}_i^n - b_j^*\|\right)  \lesssim \gO_P([\log(n)/n]^{\frac{1}{2}}),
\end{align}
for any $i \in \gA_j$, where $L_1$ is a Lipschitz constant. The above bound illustrates that the expert $f(\cdot | X; a_j^*, b_j^*)$ also enjoys the parametric estimation rate of order $\gO_P([\log(n)/n]^{\frac{1}{2}})$. 
Meanwhile, over-specified parameters whose Voronoi cells $\gA_j$ contain multiple elements are estimated at slower rates, standing at the order of $\gO_P([\log(n)/n]^{\frac{1}{4}})$. By arguing as in equation \eqref{eq:expert_rate}, we deduce that if the expert $f(\cdot | X; a_j^*, b_j^*)$ is fitted by at least two experts, then its estimation rate is of order $\gO_P([\log(n)/n]^{\frac{1}{4}})$. 
 
\vspace{0.5em}
\noindent
\textbf{Benefits of sigmoid gate over softmax gate.} For the softmax gating MLMoE model, \cite{nguyen2024general} demonstrated that if there exists $j\in[k^*]$ such that $a^*_{j\ell}=0$, for any $\ell\in[K-1]$, then the gating parameters $\alpha^*_j$'s will interact with expert parameters $a^*_{j\ell}$'s via certain partial differential equations (PDEs).
Remarkably, because of this interaction, the parameter estimation rate is slower than any polynomial rate, say $\gO_P(n^{-1/(2r)})$, for any $r \ge 1$. In contrast, when using the sigmoid gating function, such interactions no longer hold. As a consequence, parameter and expert estimation rates are significantly improved to the polynomial orders as shown in Theorem~\ref{theorem:without_temperature} and summarized in Table~\ref{tab:convergence-rate}. Hence, we claim that the sigmoid gate is more sample-efficient than the softmax gate in the MLMoE model.


\section{Sigmoid Gating Function with Temperature}
\label{sec:with_temperature}
In this section, we study the convergence properties of parameter estimation in the MLMoE model when using the sigmoid gating function with a temperature parameter. In Section~\ref{subsection:inner_product}, we first show that applying the temperature directly to the inner product affinity score will induce an intrinsic parameter interaction, leading to substantially slow convergence rates of parameter estimation. In response, we propose using a novel Euclidean affinity score in Section~\ref{subsection:euclidean_affinity} and prove that it indeed improves parameter estimation rates.

\subsection{Inner product affinity score}\label{subsection:inner_product}
To begin, we present the formulation of a sigmoid gating MLMoE with temperature parameter applied to the standard inner product affinity score.

\vspace{0.5em}
\noindent
\textbf{Problem setup.} Suppose that the data $\{(X_i,Y_i)\}_{i=1}^{n}\subset\mathbb{R}^d\times [K]$ are i.i.d. samples from a sigmoid gating MLMoE, whose conditional probability density is, for any $s \in [K]$, 
\begin{align}\label{eq:with_temperature_inner}
    \widetilde{p}_{\widetilde{G}^*}(Y=s|X)&:=\sum_{i=1}^{\ks}\frac{\exp(\gasi)\sigma(\frac{(\alsi)^{\top}X+\besi}{\tas})}{\sum_{j=1}^{\ks}\exp(\gasj)\sigma(\frac{(\alsj)^{\top}X+\besj}{\tas})}\cdot \frac{\exp((\asis)^{\top}X+\bsis)}{\sum_{\ell=1}^{K}\exp((\asil)^{\top}X+\bsil)},
\end{align}
Here, $\widetilde{G}^*:=\sum_{i=1}^{\ks}\exp(\gasi)\delta_{(\alsi,\besi,\tas,a^*_i,b^*_i)}$ denotes the true yet unknown \textit{mixing measure} associated with parameters $(\alpha_i^*, \beta_i^*, \tau^*, a_{i}^*$, $b_i^*) \in \Upsilon$, where $\Upsilon$ is a compact subset of $\mathbb{R}^d \times \mathbb{R} \times \mathbb{R} \times \mathbb{R}^{d \times K} \times \mathbb{R}^K$. In the model’s gating network, the output of the linear layer is scaled by the temperature before passing through the sigmoid function. We maintain the assumptions on the parameters outlined in Section~\ref{sec:without_temperature}, unless explicitly stated otherwise. 

\vspace{0.5em}
\noindent
\textbf{Maximum likelihood estimation.} 
The MLE of  $\widetilde{G}^*$ is given by 
\begin{equation}
    \widetilde{G}_n \in \argmax_{G \in \gG_k(\Upsilon)}\sum_{i=1}^n\log(\widetilde{p}_{G}(Y_i | X_i)),
\end{equation}
where $\widetilde{p}_{G}$ is of the form specified in equation \eqref{eq:with_temperature_inner} and $\gG_k(\Upsilon)$ denotes the set of all mixing measures with at most $k \geq k^*$ components of the form $G = \sum_{i=1}^{k'}\exp(\gamma_i)\delta_{(\alpha_i,\beta_i, \tau, a_i,b_i)}$, in which $1 \le k' \le k$ and $(\alpha_i,\beta_i, \tau, a_i,b_i) \in \Upsilon$. It is worth noting that the temperature sigmoid gating MLMoE with inner product affinity score remains identifiable. In other words, if the density estimation $\widetilde{p}_{\widetilde{G}_n}$ converges to the true density $\widetilde{p}_{\widetilde{G}^*}$, then the MLE $\widetilde{G}_n$ also converges to the true counterpart $\widetilde{G}^*$.
Similarly to Proposition~\ref{prop:density_rate}, under the Hellinger distance, the estimated density $\widetilde{p}_{\widetilde{G}_n}(Y|X)$ converges to the true density $\widetilde{p}_{\widetilde{G}^*}(Y|X)$ at the rate of order $\mathcal{O}_P(\sqrt{Kdk\log(n)/n})$. 

\vspace{0.5em}
\noindent
To obtain the parameter estimation rates for the model in equation \eqref{eq:with_temperature_inner} from the density estimation rate above, we decompose the density discrepancy $\widetilde{p}_{\widetilde{G}_n}(\cdot | X) - \widetilde{p}_{\widetilde{G}^*}(\cdot | X)$ into a sum of linearly independent components. This decomposition is achieved by applying a Taylor expansion to the product of sigmoid numerator and expert function $\widetilde{u}(Y=s|X;\alpha,\beta,\tau,a,b) := \sigma(\frac{\alpha^\top X + \beta}{\tau})f(Y=s|X;a,b)$. 
Unfortunately, when using the temperature $\tau$ in the sigmoid gate, there exists an intrinsic interaction with the gating parameters $\alpha,\beta$, captured by the PDE
\begin{equation}\label{eq:interaction}
    \dfrac{\partial \widetilde{u}}{\partial \tau} = - \dfrac{1}{\tau} \cdot \alpha^\top \dfrac{\partial \widetilde{u}}{\partial \alpha} -  \beta \dfrac{\partial \widetilde{u}}{\partial \beta}.
\end{equation}
Consequently, we proceed to construct a Voronoi loss function to analyze how this interaction influences the parameter estimation rate in Theorem~\ref{theorem:with_temperature_inner}.

\vspace{0.5em}
\noindent
\textbf{Voronoi loss.} For any mixing measure $G$ in $\gG_k(\Upsilon)$ and non-negative integers $\{ \kappa_l, l =1,\ldots,5\}$, let us define, using the notation introduced in the previous section,
$K_{ij}(\kappa_1,\kappa_2,\kappa_3,\kappa_4,\kappa_5):= \|\Delta\alpha_{ij}\|^{\kappa_1} + |\Delta\beta_{ij}|^{\kappa_2}  +|\Delta\tau|^{\kappa_3}+\sum_{\ell=1}^K
\big(\|\Delta a_{ij\ell}\|^{\kappa_4} + |\Delta b_{ij\ell}|^{\kappa_5}\big)$, where we denote $\Delta\tau := \tau - \tau^*$. Then,
for each $r \ge 1$, the Voronoi loss of interest is given by 
\begin{align}\label{eq: d_{3,r}}
    \gD_{2,r}(G,\widetilde{G}^*)
&:= \sum_{j=1}^{\ks}
    \left|\sum_{i\in\gA_j} \exp(\gamma_i) - \exp(\gasi)\right|+ \sum_{j=1}^{\ks}\sum_{i\in\gA_j}
    \exp(\gamma_i)K_{ij}(r,r,r,r,r).
\end{align}
The PDE \eqref{eq:interaction} produces notable and undesirable effects that the Hellinger distance $\mathbb{E}_X[d_H(\widetilde{p}_{\widetilde{G}_n}(\cdot | X),\widetilde{p}_{\widetilde{G}^*}(\cdot | X) )]$ is no longer lower bounded by $  \gD_{2,r}(\widetilde{G}_n, \widetilde{G}^*)$. As a result, we cannot use the above arguments to conclude that $\gD_{2,r}(\widetilde{G}_n, \widetilde{G}^*) = \gO_P(n^{-1/2})$ as discussed in Theorem~\ref{theorem:without_temperature}. Instead, we show in Appendix~\ref{appendix:with_temperature_inner} that 
$$
\inf_{G \in \gG_k(\Upsilon): \gD_{2,r}(G, \widetilde{G}^*)\le \varepsilon}\dfrac{\mathbb{E}_X[d_V(\widetilde{p}_G(\cdot|X), \widetilde{p}_{\widetilde{G}^*}(\cdot |X))]}{\gD_{2,r}(G, \widetilde{G}^*)} \to 0,
$$
when $\varepsilon \to 0$ for any $r \ge 1$. 

\noindent
The above result implies the subsequent minimax lower bound of parameter estimation:
\begin{theorem}\label{theorem:with_temperature_inner}
    Under the over-specified settings, the following minimax lower bound of estimating $\widetilde{G}^*$ holds true for any $r \ge 1$:
    $$
        \inf_{{G}_n \in \gG_k(\Upsilon)}\sup_{G \in \gG_k(\Upsilon)\setminus \gG_{\ks-1}(\Upsilon)}\mathbb{E}_{\widetilde{p}_G}[\gD_{2,r}(G_n,G)] \gtrsim n^{-1/2}.
    $$
    Here, the notation $\mathbb{E}_{\widetilde{p}_G}$ indicates the expectation taken w.r.t. the product measure with mixture density $\overline{p}_G^n$.  
\end{theorem}

\noindent
The proof of Theorem~\ref{theorem:with_temperature_inner} is in Appendix~\ref{appendix:with_temperature_inner}. The above minimax lower bound together with the formulation of $\gD_{2,r}$ in equation \eqref{eq: d_{3,r}} indicates that the estimation rates for parameters $\alpha_j^*, \beta_j^*, \tau^*, a_j^*, b_j^*$ are slower than $\gO_P(n^{-1/(2r)})$ for any $r \ge 1$. This implies that these rates are slower than any polynomial rate and may decay as slowly as $\mathcal{O}_P(1/\log(n))$. 


\vspace{0.5em}
\noindent
Using the same argument for Lipschitz function described after equation \eqref{eq:expert_rate}, 
it follows that the rates for estimating experts $f(\cdot | X;a_j^*, b_j^*)$ are no better than those for estimating the parameters $a_j^*$ and $b_j^*$, and could also be as slow as $\gO_P(1/\log(n))$. This result suggests that incorporating a temperature term into the standard sigmoid gate in MLMoE models is not sample-efficient, as it may produce undesirable effects, as demonstrated in our analysis.

\subsection{Euclidean affinity score}\label{subsection:euclidean_affinity}
To eliminate the undesirable interaction between the inner product affinity score and the temperature parameter observed in the PDE~\eqref{eq:interaction}, we propose replacing it with a novel Euclidean affinity score.  We show that this Euclidean affinity formulation ensures a parameter estimation rate of polynomial order for the temperature sigmoid gating MLMoE. The formulation of this model is defined as follows:

\vspace{0.5em}
\noindent
\textbf{Problem setup.} Again, we assume i.i.d. data $\{(X_i,Y_i)\}_{i=1}^{n}\subset\mathbb{R}^d\times[K]$, but this time from a modified sigmoid gating multinomial logistic mixture of experts with conditional density function given by, for each $s \in [K]$, 
\begin{align}\label{eq:with_temperature_euclidean}
    \overline{p}_{\overline{G}^*}(Y=s|X)&:=\sum_{i=1}^{\ks}\frac{\exp(\gasi)\sigma(\frac{\|\alsi-X\|+\besi}{\tas})}{\sum_{j=1}^{\ks}\exp(\gasj)\sigma(\frac{\|\alsj-X\|+\besj}{\tas})} \cdot \frac{\exp((\asis)^{\top}X+\bsis)}{\sum_{\ell=1}^{K}\exp((\asil)^{\top}X+\bsil)},
\end{align}
As usual, $\overline{G}^*:=\sum_{i=1}^{\ks}\exp(\gasi)\delta_{(\alsi,\besi,\tas,a^*_i,b^*_i)}$ denotes the true yet unknown \textit{mixing measure}, which is the combination of Dirac measures $\delta$ associated with true parameters $(\alpha_i^*, \beta_i^*, \tau^*, a_{i}^*$, $b_i^*) \in \Upsilon$. Denote $\overline{u}(X;\alpha,\beta,\tau,a,b) := \sigma\left(\frac{\|\alpha-X\|+\beta}{\tau}\right)f(Y=s|X;a,b)$ . Then, the first-order independence condition on the Euclidean affinity score guarantees that the interaction between gating parameters $\alpha_i,\beta_i$ and temperature $\tau$ no longer holds, that is 
\begin{equation}
    \dfrac{\partial \overline{u}}{\partial \tau} \ne - \dfrac{1}{\tau} \cdot \alpha^\top \dfrac{\partial \overline{u}}{\partial \alpha} -  \beta \dfrac{\partial \overline{u}}{\partial \beta}.
\end{equation}
As a result, the estimation rate for parameters $\alpha_j^*, \beta_j^*, \tau^*, a_j^*, b_j^*$ should be improved in comparison with those in Section~\ref{subsection:inner_product}.

\vspace{0.5em}
\noindent
\textbf{Maximum likelihood estimation.} According to the change of gating function, let us re-define the MLE corresponding to the model \eqref{eq:with_temperature_euclidean} as 
$$
\overline{G}_n:=\argmax_{G \in \gG_k(\Upsilon)}\sum_{i=1}^n\log(\overline{p}_G(Y|X)).
$$ 
By following the same line of reasoning used in the inner product affinity score case, the conditional density $\overline{p}_{\overline{G}}(\cdot | X)$ remains identifiable. Moreover, the rate that the density estimator $\overline{p}_{\overline{G}_n}(\cdot | X)$ converges to $\overline{p}_{\overline{G}^*}(\cdot | X)$ remains $\gO_P(\sqrt{Kdk\log(n)/n})$. Based on this observation, if we can construct a loss $\gD_3$ among parameters such that $\bbE_X[d_{H}(\overline{p}_{\overline{G}_n}(\cdot|X),\overline{p}_{\overline{G}^*}(\cdot|X))] \gtrsim \gD_3(\overline{G}_n, \overline{G}^*)$, then the MLE $\overline{G}_n$ will also converge to the true mixing measure $\overline{G}^*$ at the rate of order $\gO_P(n^{-1/2})$. 

\vspace{0.5em}
\noindent
\textbf{Voronoi loss.} We propose the Voronoi loss
\begin{align}
    \gD_3(G, \overline{G}^*) &:= \sum_{j=1}^{\ks}
    \left|\sum_{i\in\gA_j} \exp(\gamma_i) - \exp(\gasi)\right| \nonumber\\
& \quad+ \sum_{\substack{j: |\gA_j| >1, \\i\in\gA_j}}
    \exp(\gamma_i)K_{ij}(2,2,2,2,2)+ \sum_{\substack{j: |\gA_j| =1, \\i\in\gA_j}}\exp(\gamma_i)K_{ij}(1,1,1,1,1).
\end{align}
In the following theorem, we provide the parametric convergence rate under the temperature sigmoid gating MLMoE model with euclidian affinity score. 
\begin{theorem}    \label{theorem:with_temperature_euclidean}
    The following Hellinger lower bound holds true for any $G \in \gG_{k}(\Upsilon)$:
    $$
    \mathbb{E}_X[d_H(\overline{p}_G(\cdot | X), \overline{p}_{\overline{G}^*}(\cdot | X))] \gtrsim \gD_3(G, \overline{G}^*).
    $$
    This bound leads to the parametric convergence rate of MLE: $\gD_3(\overline{G}_n, \overline{G}^*) = \mathcal{O}_P(\sqrt{Kdk\log(n)/n})$. 
\end{theorem}

\noindent
The proof of Theorem~\ref{theorem:with_temperature_euclidean} is in Appendix~\ref{appendix:with_temperature_euclidean}. This result reveals that parameters $\alpha_j^*, \beta_j^*, \tau^*, a_j^*, b_j^*$ associated with a single fitted component  ($|\gA_j(\overline{G}_n)| = 1$) achieve the estimation rate of order $\gO_P([\log(n)/n]^{\frac{1}{2}})$. In contrast, parameters that are approximated by multiple components  ($|\gA_j(\overline{G}_n)| > 1$) converge more slowly at the rate $\gO_P([\log(n)/n]^{\frac{1}{4}})$. These rates match those in the sigmoid gating MLMoE without temperature parameter established in Theorem~\ref{theorem:without_temperature}.

\vspace{0.5em}
\noindent
Moreover, using the same argument following equation \eqref{eq:expert_rate}, we obtain that the expert estimators also converge at the rate of order $\gO_P([\log(n)/n]^{\frac{1}{2}})$ when fitted by a single expert, and at the rate of order $\gO_P([\log(n)/n]^{\frac{1}{4}})$ when approximated by multiple experts. Therefore, replacing the standard inner product affinity score with our proposed Euclidean affinity score substantially improves the expert estimation rate, boosting it from approximately $\gO_P(1/\log(n))$ to polynomial order. This highlights the advantages of the Euclidean affinity score over the inner product operation. 


\section{Numerical Experiments}
\label{sec:experiments}
In this section, we perform numerical experiments to empirically demonstrate our theoretical results on the convergence behavior of the sigmoid gating MLMoE model.

\vspace{0.5em}
\noindent
\textbf{Synthetic data.} For each sample size $n$, we generate independent and identically distributed samples $\{(X_i, Y_i)\}_{i=1}^n$ by first drawing $X_i$'s from the uniform distribution over $[0,1]$ and then sampling $Y_i$'s from the true conditional density $p_{\Gs}(Y=s|X)$, $\widetilde{p}_{\widetilde{G}^*}(Y=s|X)$, and $\overline{p}_{\overline{G}^*}(Y=s|X)$ corresponding to the MoE model specified in each theorem configuration. 

\vspace{0.5em}
\noindent
\textbf{Maximum likelihood estimation.} A widely used approach for computing the maximum likelihood estimator $\widehat{G}$ (or $\widetilde{G}, \overline{G}$)
for each set of samples is to use the Expectation-Maximization (EM) algorithm \cite{dempster1977maximum}. Since closed form updates for the gating parameters are not available in the maximization step, we instead employ an EM based numerical procedure previously adopted in \cite{chamroukhi2009time}. We set the convergence tolerance to $\varepsilon = 10^{-6}$ and run the algorithm for a maximum of 1000 iterations.  

\vspace{0.5em}
\noindent
\textbf{Experimental design.} Our empirical study considers three experimental configurations, each directly aligned with the theoretical settings developed in the main paper. For each configuration, we generate 20 independent datasets across 20 distinct sample sizes ranging from $10^4$ to $10^5$. To ensure consistency and facilitate comparison across experiments, we adopt a model architecture with two true experts $(\ks=2)$ and two classes ($K=2$), while fitting models with three and four experts in the experimental settings.

\vspace{0.5em}
\noindent
\textbf{Initialization.} For each $k \in \{\ks+1, \ks+2\}$, the indices in the set ${1,2,...,k}$ are randomly assigned to $\ks$ Voronoi cells denoted by $\gA_1, \ldots,\gA_{\ks}$,  ensuring that each cell contains at least one element. This assignment is repeated independently for each replication. Subsequently, the model parameters are initialized by sampling from Gaussian distributions centered at their true values with a small variance to ensure stable convergence.

\subsection{Sigmoid gate versus Softmax gate in the MLMoE} 


As illustrated in Figures~\eqref{fig:subfig_compare_k3} and \eqref{fig:subfig_compare_k4}, when the sigmoid gating is employed without a temperature parameter, the Voronoi loss converges to zero at rates of order $\gO(n^{-0.49})$ and $\gO(n^{-0.54})$ when fitting $k=3$ and $k=4$ experts, respectively. 
These empirical rates align with the theoretical guarantees provided in Theorem~\ref{theorem:without_temperature}. 

\vspace{0.5em}
\noindent
We further investigate the convergence rate when replacing sigmoid gating with the softmax gating configuration in \cite{nguyen2024general}. In both cases, the parameters in the sigmoid gating model converge faster than those in the softmax gating model, achieving $\gO(n^{-0.49})$ versus $\gO(n^{-0.30})$ for $k=3$, and $\gO(n^{-0.54})$ versus $\gO(n^{-0.28})$ for $k=4$. These empirical results are consistent with the claim in Section~\ref{sec:without_temperature}, underscoring the role of the sigmoid gating function in improving the model's sample efficiency. 

\subsection{Sigmoid gating function with temperature}


When the temperature parameter is introduced to the inner product affinity score used in the standard sigmoid gating MLMoE, the convergence rate becomes significantly slower. 
As illustrated in Figures~\eqref{fig:conv_d2_k3} and \eqref{fig:conv_d2_k4}, the Voronoi loss decreases at rates of order $\gO(n^{-0.11})$ and $\gO(n^{-0.15})$ for $k=3$ and $k=4$, respectively, which aligns with the theoretical results in Theorem~\ref{theorem:with_temperature_inner}.

\vspace{0.5em}
\noindent
In contrast, when the Euclidean affinity score is adopted in the temperature sigmoid gating MLMoE, the convergence behavior improves substantially. 
As shown in Figures~\eqref{fig:conv_d3_k3} and \eqref{fig:conv_d3_k4}, the Voronoi loss converges at rates of order $\gO(n^{-0.52})$ and $\gO(n^{-0.50})$ for $k=3$ and $k=4$, respectively. 
These empirical evidences provide strong support for the theoretical results in Theorem~\ref{theorem:with_temperature_euclidean}.

\begin{figure*}[t]
\centering

\begin{subfigure}[t]{0.30\textwidth}
  \centering
  \includegraphics[width=\linewidth]{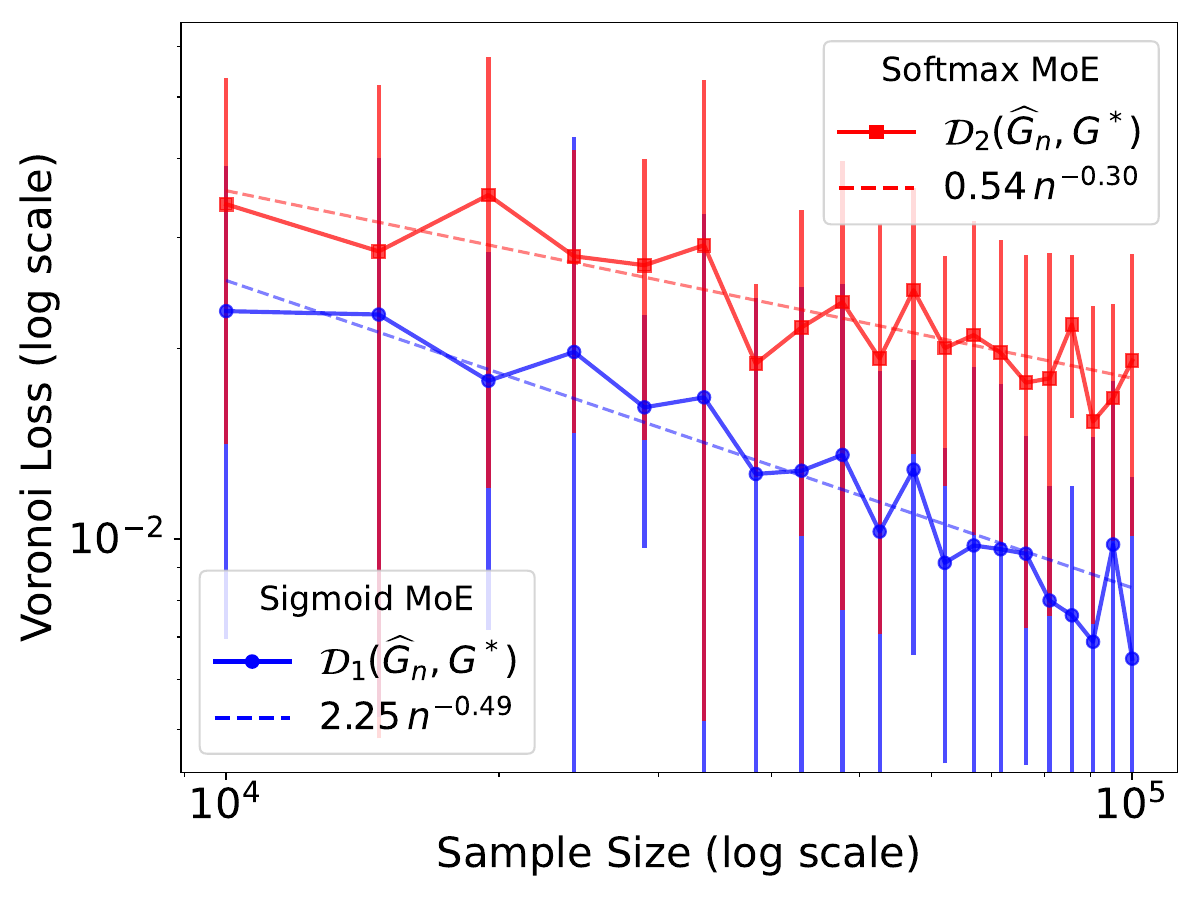}
  \caption{Convergence rate of $\gD_1(\hat{G}_n,{G}^*)$ vs $\gD_2(\hat{G}_n,{G}^*)$ when $k = 3$}
  \label{fig:subfig_compare_k3}
\end{subfigure}
\hfill
\begin{subfigure}[t]{0.30\textwidth}
  \centering
  \includegraphics[width=\linewidth]{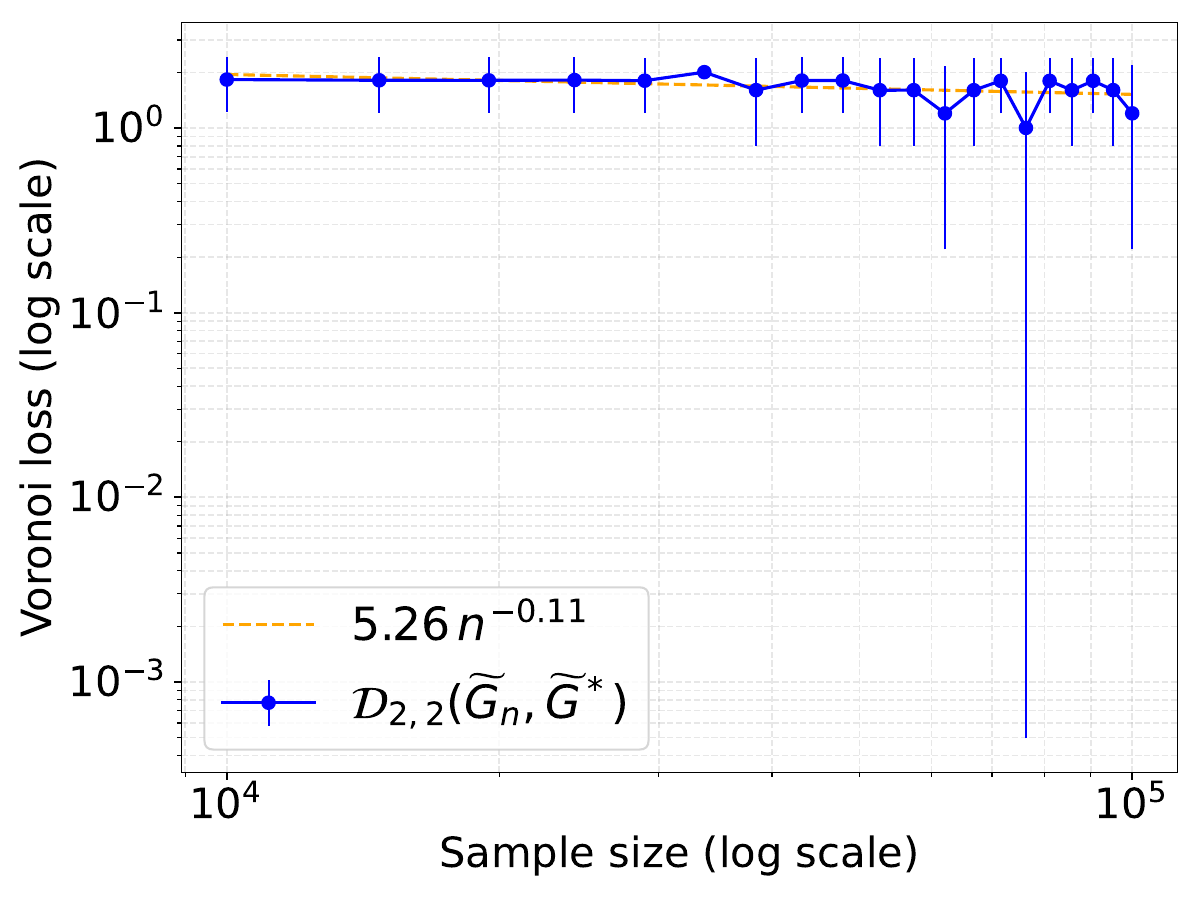}
  \caption{Convergence rate of $\gD_{2,2}(\widetilde{G}_n,\widetilde{G}^*)$ when $k = 3$}
  \label{fig:conv_d2_k3}
\end{subfigure}
\hfill
\begin{subfigure}[t]{0.30\textwidth}
  \centering
  \includegraphics[width=\linewidth]{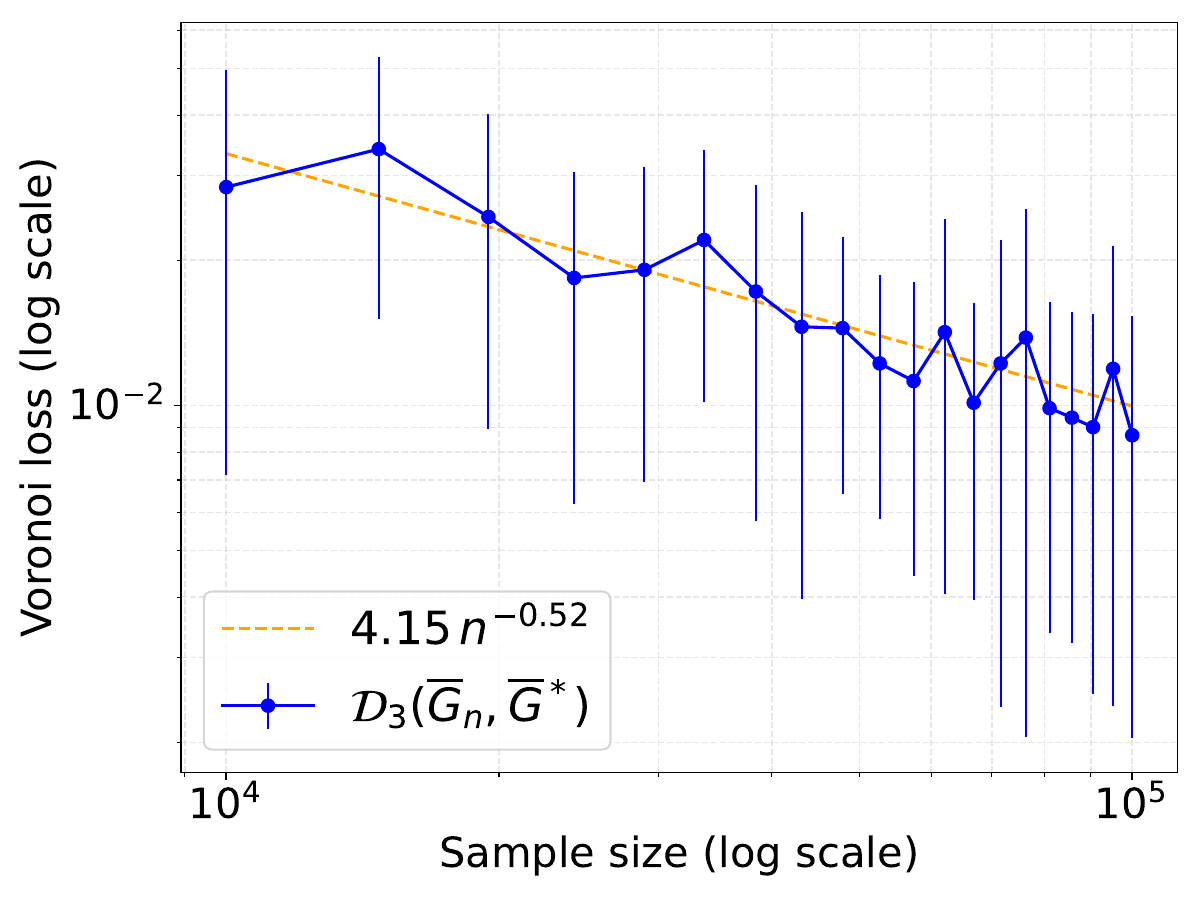}
  \caption{Convergence rate of $\gD_3(\overline{G}_n,\overline{G}^*)$ when $k = 3$}
  \label{fig:conv_d3_k3}
\end{subfigure}

\vspace{0.1cm}

\begin{subfigure}[t]{0.30\textwidth}
  \centering
  \includegraphics[width=\linewidth]{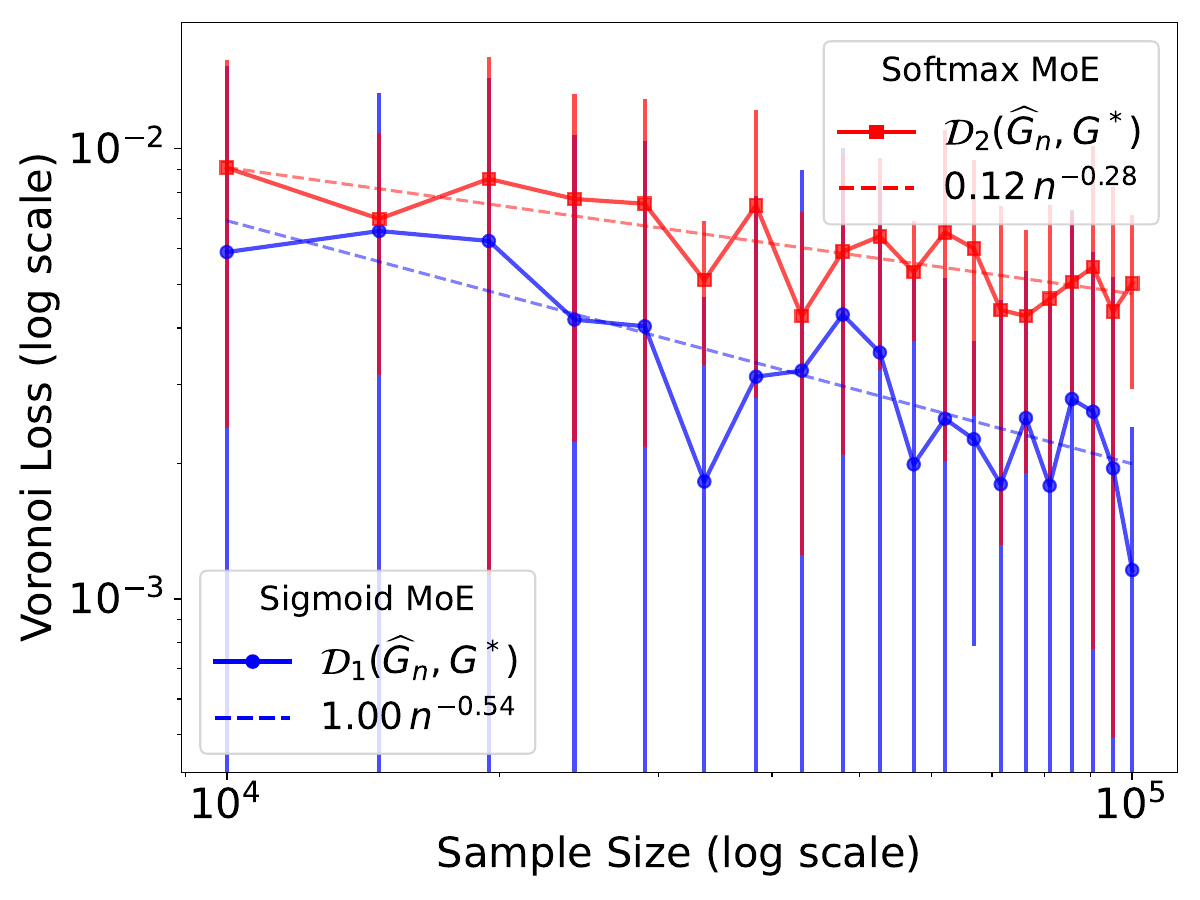}
  \caption{Convergence rate of $\gD_1(\hat{G}_n,{G}^*)$ vs $\gD_2(\hat{G}_n,{G}^*)$ when $k = 4$}
  \label{fig:subfig_compare_k4}
\end{subfigure}
\hfill
\begin{subfigure}[t]{0.30\textwidth}
  \centering
  \includegraphics[width=\linewidth]{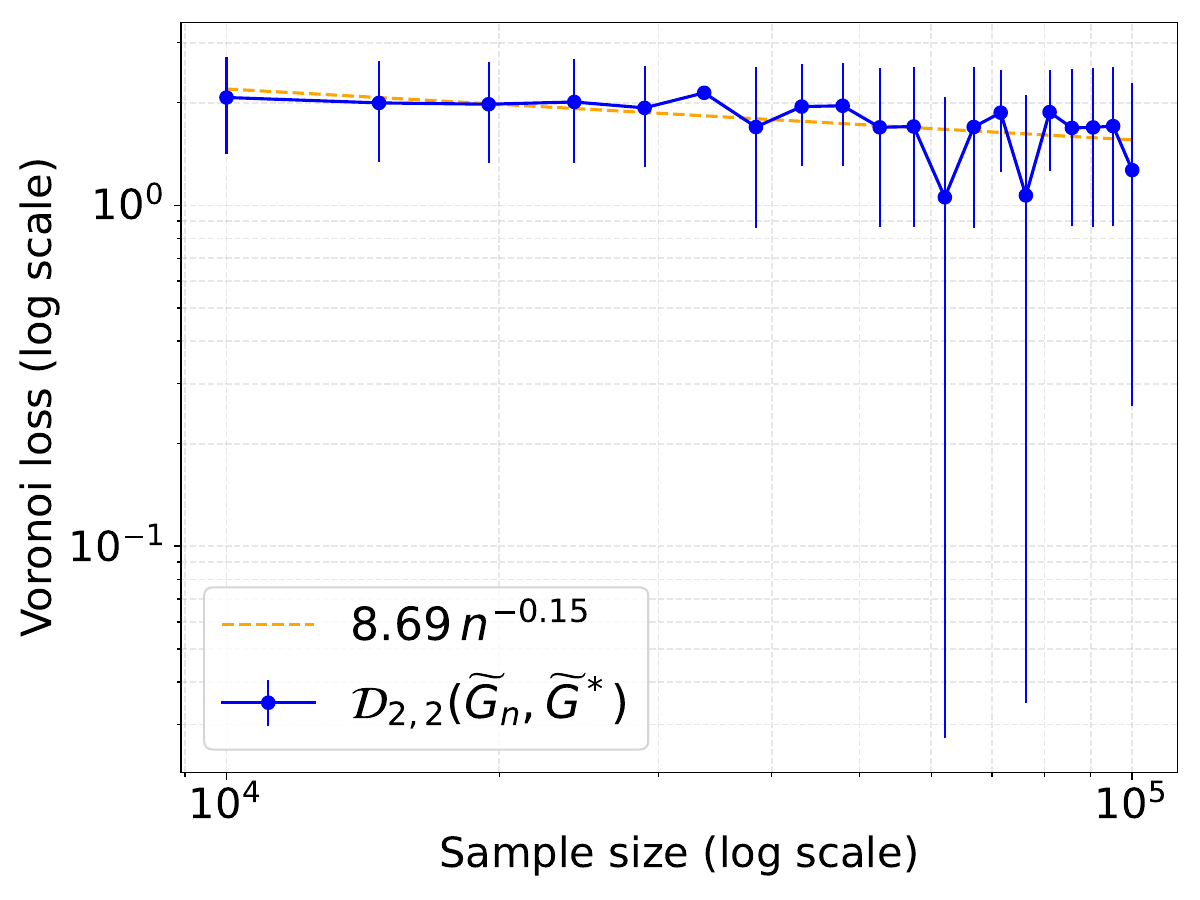}
  \caption{Convergence rate of $\gD_{2,2}(\widetilde{G}_n,\widetilde{G}^*)$ when $k = 4$}
  \label{fig:conv_d2_k4}
\end{subfigure}
\hfill
\begin{subfigure}[t]{0.30\textwidth}
  \centering
  \includegraphics[width=\linewidth]{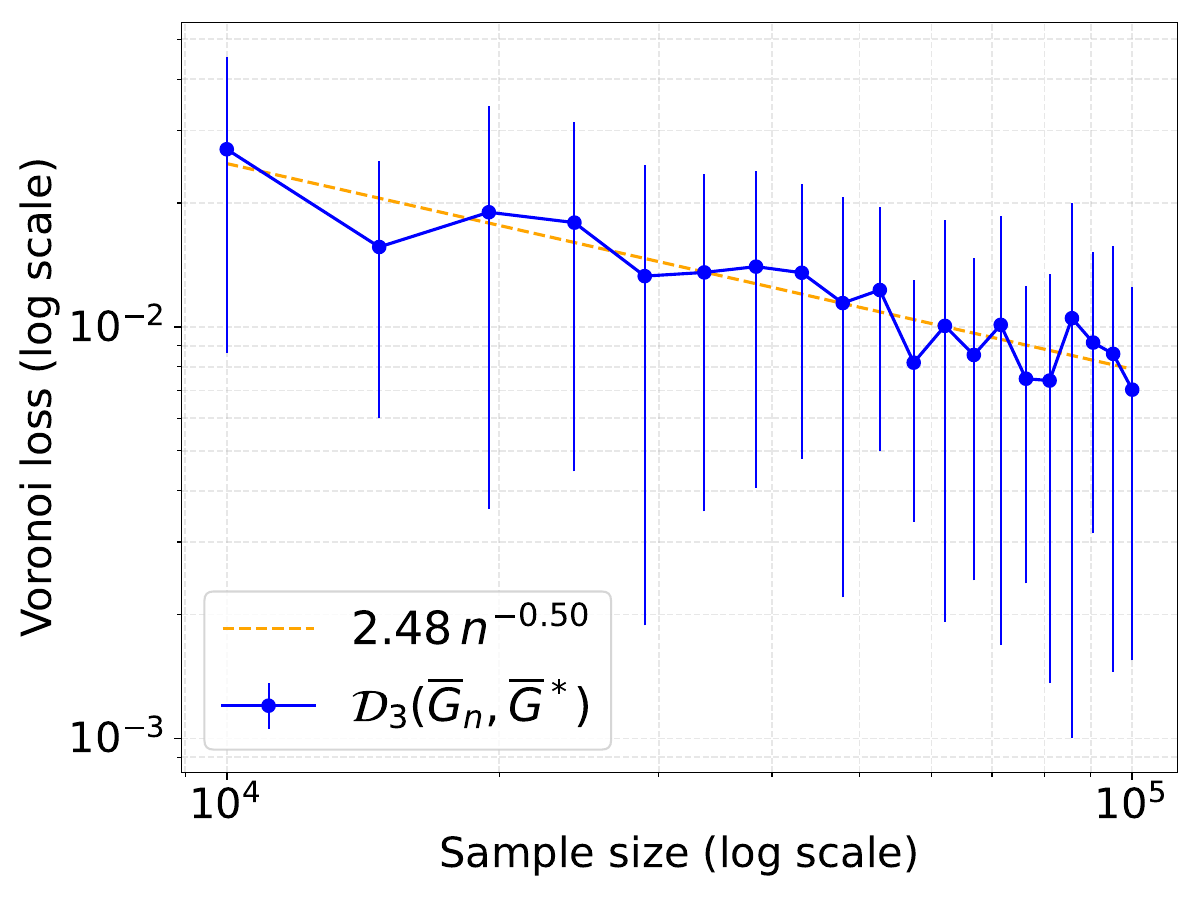}
  \caption{Convergence rate of $\gD_3(\overline{G}_n,\overline{G}^*)$ when $k = 4$}
  \label{fig:conv_d3_k4}
\end{subfigure}

\captionsetup{justification=justified,singlelinecheck=false}
\caption{Log-log scaled plots for the empirical convergence rates. Left column: comparison between sigmoid gating without temperature ($G_n$) and softmax gating for the Voronoi metric. Middle and right columns: convergence rates of the MLE $\widetilde{G}_n$ and $\overline{G}_n$ for inner product affinity score and Euclidean affinity score settings, respectively. In these figures, the corresponding empirical discrepancies are illustrated by the blue curves for sigmoid gating, while the red curves represent softmax gating. Orange dash-dotted lines indicate the least-squares fitted linear regression lines. Error bars represent two times the empirical standard deviation.}
\label{fig:convergence_rate}
\end{figure*}

\section{Discussion}
\label{sec:discussion}
In this paper, we investigate the convergence behavior of parameter and expert estimation in the sigmoid-gated MLMoE models. 
We first introduce a positive scale parameter in the standard sigmoid gating function to resolve its fundamental limitation in over-specified settings of MoE models where it fails to ensure the density convergence. Given this modification, we demonstrate that the sigmoid gate is more sample-efficient than the softmax gate in terms of estimating parameters and experts in the MLMoE.
We further characterize the role of the temperature parameter in the sigmoid gate with inner product affinity score. We find that its interaction with gating parameters leads to slow expert estimation rates of non-polynomial order. By using a novel Euclidean affinity score as an alternative, we show that the previous interaction is eliminated, thereby restoring the polynomial-order convergence rates of expert estimation. 

\vspace{0.5em}
\noindent
\textbf{Practical implications.} 
Our theories yield three important practical implications on the design of MLMoE models.

\vspace{0.5em}
\noindent
\emph{(I.1) Sigmoid gate is provably more sample-efficient than softmax gate in MoE for classification.}  
As shown in Table~\ref{tab:convergence-rate}, under the over-specified setting (the common practical scenario where the number of fitted experts exceeds the true number), the sigmoid gating function yields significantly faster rates for estimating parameters and experts than the softmax gating function.

\vspace{0.5em}
\noindent
\emph{(I.2) Involving a temperature parameter in the inner product affinity score reduces the model's sample efficiency.}
Our analysis shows that combining a temperature parameter with the conventional inner product affinity score creates an intrinsic interaction with gating parameters. This interaction forces the estimation rate to be slower than any polynomial order. Consequently, it implies that to achieve a given accuracy in estimating model parameters or experts, the required number of samples may grow exponentially rather than polynomially as in the absence of the temperature. 

\vspace{0.5em}
\noindent
\emph{(I.3) Euclidean affinity score is suitable for temperature-controlled gating.}  
Replacing the inner product operation with a Euclidean affinity score removes the detrimental interaction between the temperature and gating parameters and, thus, restores polynomial convergence rates. This provides a principled guideline for incorporating temperature control into the sigmoid gating function while maintaining reliable parameter estimation.

\vspace{0.5em}
\noindent
\textbf{Limitations and Future Directions.} Regarding the model formulation, in the MLMoE in equation~\eqref{eq:standard_density}, experts are set to be a multinomial logistic function of a simple linear network. In future development, we can generalize the results of this paper to the scenario in which the linear network is substituted by an arbitrary neural network. Regarding the problem setting, our analysis is conducted under a well-specified setting, where the data are generated from a sigmoid-gated MLMoE model. In practice, the true conditional distribution may lie outside this model class, leading to a misspecified setting in which the MLE converges to a pseudo-true mixing measure that minimizes the Kullback--Leibler divergence over a non-convex parameter space, where uniqueness and convergence rates are difficult to characterize. 
Moreover, under misspecification, the geometric foundation of our proof--particularly the Voronoi-cell decomposition used to control mixing-measure discrepancies--may break down, since pseudo-true parameters need not correspond to well-separated experts. Thus, developing a theory for sigmoid-gated MLMoE under misspecification remains an open problem.

\appendix

\vspace{1cm}
\centering
\textbf{\Large{Supplementary Material for 
``Rethinking Multinomial Logistic Mixture of Experts with Sigmoid Gating
Function''}}

\justifying
\setlength{\parindent}{0pt}
\vspace{0.5cm}

In this supplementary material, we present rigorous proofs for our theoretical results. The convergence rates of parameter estimation for the sigmoid-gated multinomial logistic mixture of experts under multiple settings are rigorously established in Appendix~\ref{appendix:main_results}. Supporting proofs for the auxiliary results on density estimation rates and model identifiability are provided in the Appendix~\ref{appendix:auxiliary_results}. Finally, we detail the experimental settings and configurations used in our numerical studies in Appendix~\ref{appendix:numerical}.

\section{Proofs of Main Results}
\label{appendix:main_results}

\subsection{Proof of Theorem~\ref{theorem:without_temperature}}
\label{appendix:without_temperature}

To establish the result in Theorem~\ref{theorem:without_temperature}, it suffices to demonstrate that
\[
\inf_{G \in \mathcal{G}_k(\Theta)}
\left(
\mathbb{E}_X\!\left[
d_H\!\bigl(
p_G(\cdot \mid X),\,
p_{G^*}(\cdot \mid X)
\bigr)
\right]
\,/\, \gD_1(G, G^*)
\right)
> 0.
\]

Given that the Total Variation distance is upper bounded by the Hellinger distance ($d_H(p,q) \gtrsim d_V(p,q)$), it suffices to establish: 
\[
\liminf_{\substack{
G \in \mathcal{G}_k(\Theta)
}}
\left(
\mathbb{E}_X\!\left[
d_V\!\bigl(
p_G(\cdot \mid X),\,
p_{G^*}(\cdot \mid X)
\bigr)
\right]
\,/\, \gD_1(G, G^*)
\right)
> 0.
\]

\noindent\textbf{Local Structure.} 
We first examine the local behavior of this inequality through a proof by contradiction:
\[
\liminf_{\substack{
\epsilon \to 0 \\
G \in \mathcal{G}_k(\Theta) \\
\gD_1(G, G^*) \le \epsilon
}}
\left(
\mathbb{E}_X\!\left[
d_V\!\bigl(
p_G(\cdot \mid X),\,
p_{G^*}(\cdot \mid X)
\bigr)
\right]
\,/\, \gD_1(G, G^*)
\right)
> 0.
\]

Suppose, for the sake of contradiction, that the above does not hold. 
Then, there must exist a sequence of mixing measures
$$
G_n := \sum_{j=1}^{k_n} \exp(\gamma_j^n) \, \delta_{(\alpha_j^n, \beta_j^n, a_j^n, b_j^n)} 
\in \mathcal{G}_k(\Theta)
$$
such that both
$\mathbb{E}_X\!\left[
d_V\!\bigl(
p_{G_n}(\cdot \mid X),\,
p_{G^*}(\cdot \mid X)
\bigr)
\right]
\,/\, \gD_1(G_n, G^*)
 \text{ and }
\gD_1(G_n, G^*)$
converge to zero as $n \to \infty$. 
We now proceed by splitting the remaining argument into three distinct steps.

\noindent\textbf{Step 1.} 
In this step, we apply a Taylor expansion to decompose the difference
$$
T_n(s) := \left[\sum_{j=1}^{k_*} \exp(\gamma^*_{j})\sigma((\alpha_j^*)^{\top}X + \beta_j^*) \right]\big[p_{G_n}(Y=s|X) - p_{G^*}(Y=s|X)\big].
$$
We define
\begin{align*}
    u(Y=s|X; \alpha, \beta, a, b) &:= \sigma(\alpha^\top X + \beta) f(Y=s|X; a, b); \\
    v(Y=s|X; \alpha, \beta) &:= \sigma(\alpha^\top X + \beta) p_{G_n}(Y=s|X),
\end{align*}
for each $s \in [K]$. 
Then $T_n(s)$ can be expressed as
\begin{align}
T_n(s) 
&= \sum_{j=1}^{k_*} \sum_{i \in \gA_j} \exp(\gamma^n_{i}) 
\left[ 
u(Y=s|X; \alpha^n_i, \beta^n_i, a^n_i, b^n_i) 
- 
u(Y=s|X; \alpha^*_j, \beta^*_j, a^*_j, b^*_j) 
\right] \nonumber \\
&\quad - \sum_{j=1}^{k_*} \sum_{i \in \gA_j} \exp(\gamma^n_{i}) 
\left[ 
v(Y=s|X; \alpha^n_i, \beta^n_i) 
- 
v(Y=s|X; \alpha^*_j, \beta^*_j) 
\right] \nonumber \\
&\quad + \sum_{j=1}^{k_*} 
\left( 
\sum_{i \in \gA_j} \exp(\gamma^n_{i}) - \exp(\gamma^*_{j}) 
\right)
\left[ 
u(Y=s|X; \alpha^*_j, \beta^*_j, a^*_j, b^*_j) 
- 
v(Y=s|X; \alpha^*_j, \beta^*_j) 
\right] \nonumber \\
&:= A_n - B_n + E_n.
\end{align}

Based on the expressions for $A_n$ and $B_n$, we further decompose each according to the size of the Voronoi cells $C_j$. In particular,
\begin{align*}
A_n &= \sum_{j: |\gA_j| = 1} \sum_{i \in \gA_j} \exp(\gamma_i^n) 
\left[ u(Y=s|X; \alpha_i^n, \beta_i^n, a_i^n, b_i^n) - u(Y=s|X; \alpha_j^*, \beta_j^*, a_j^*, b_j^*) \right] \\
&\quad + \sum_{j: |\gA_j| > 1} \sum_{i \in \gA_j} \exp(\gamma_i^n) 
\left[ u(Y=s|X; \alpha_i^n, \beta_i^n, a_i^n, b_i^n) - u(Y=s|X; \alpha_j^*, \beta_j^*, a_j^*, b_j^*) \right] \\
&= A_{n,1} + A_{n,2}.
\end{align*}

Next, we define $h_\ell(X; a_{j\ell}, b_{j\ell}) := (a_{j\ell})^{\top}X+b_{j\ell}, \forall\, \ell \in [K]
$. Then, 
\[
f(Y=s|X; a_j, b_j) = \frac{\exp((a_{js})^{\top}X+b_{js})}{\sum_{\ell=1}^{K}\exp((a_{j\ell})^{\top}X+b_{j\ell})} = \frac{\exp(h_s(X; a_{js}, b_{js}))}{\sum_{\ell=1}^{K}\exp(h_\ell(X; a_{js}, b_{js}))}.
\]
Using a first-order Taylor expansion, $A_{n,1}$ can be written as
\begin{align*}
A_{n,1} &= \sum_{j: |\gA_j|=1} \sum_{i \in \gA_j} \exp(\gamma_i^n) \sum_{|r|=1} \frac{1}{r!} 
(\Delta \alpha_{ij}^n)^{r_1}(\Delta \beta_{ij}^n)^{r_2} \prod_{\ell=1}^{K} (\Delta a_{ij\ell}^n)^{r_{3\ell}} (\Delta b_{ij\ell}^n)^{r_{4\ell}} \\
&
\qquad  \times\frac{\partial^{|r|} u}{\partial \alpha_j^{r_1} \partial \beta_j^{r_2}\prod_{\ell=1}^{K} \partial a_{j\ell}^{r_{3\ell}} \partial b_{j\ell}^{r_{4\ell}}}(Y=s|X; \alpha_j^*, \beta_j^*, a_j^*, b_j^*) + R_1(X,Y),
\end{align*}
where
$
r = (r_1, r_2, r_{31}, \dots, r_{3K}, r_{41}, \dots, r_{4K}),
$
with $r_1,r_{3\ell} \in \mathbb{N}^d$ and $r_2, r_{4\ell} \in \mathbb{N}$ for all $\ell \in [K]$. Additionally, $R_1(X,Y)$ is the Taylor remainder satisfying $R_1(X,Y)/\gD_1(G_n, G^*) \to 0$ as $n \to \infty$.

Let 
$
g(X; \alpha_{\ell}, \beta_{\ell}) := (\alpha_{\ell})^{\top}X+\beta_{\ell}, \forall \ell \in [\ks]
$, we can further write
\begin{align*}
A_{n,1} &= \sum_{j: |\gA_j|=1} \sum_{i \in \gA_j} \exp(\gamma_i^n) \sum_{|r|=1} \frac{1}{\alpha!} (\Delta \alpha_{ij}^n)^{r_1} (\Delta\beta_{ij}^n)^{r_2}\prod_{\ell=1}^{K} (\Delta a_{ij\ell}^n)^{r_{3\ell}} (\Delta b_{ij\ell}^n)^{r_{4\ell}} \\
&\quad \times \frac{\partial^{|r_1|+r_2} \sigma((\alpha_j^*)^{\top}X + \beta_j^*)}{\partial \alpha_j^{r_1} \partial \beta_j^{r_2}} \times\frac{\partial^{\sum_{\ell=1}^K|r_{3\ell}| + r_{4\ell}} f}{\prod_{\ell=1}^{K} \partial a_{j\ell}^{r_{3\ell}} \partial b_{j\ell}^{r_{4\ell}}}(Y=s|X; a_j^*, b_j^*) + R_1(X,Y)\\
 &= \sum_{j: |\gA_j|=1} \sum_{i \in \gA_j} \exp(\gamma_i^n) \sum_{|r|=1} \frac{1}{\alpha!} (\Delta \alpha_{ij}^n)^{r_1} (\Delta\beta_{ij}^n)^{r_2}\prod_{\ell=1}^{K} (\Delta a_{ij\ell}^n)^{r_{3\ell}} (\Delta b_{ij\ell}^n)^{r_{4\ell}} \\
&\quad \times X^{r_1+\sum_{\ell=1}^K r_{3\ell}} \times\frac{\partial^{|r_1|+r_2} \sigma((\alpha_j^*)^{\top}X + \beta_j^*)}{\partial {g}^{|r_1|+r_2}} \times\frac{\partial^{\sum_{\ell=1}^K|r_{3\ell}| + r_{4\ell}} f}{\prod_{\ell=1}^{K} \partial h_{\ell}^{|r_{3\ell}|+r_{4\ell}}}(Y=s|X; a_j^*, b_j^*) + R_1(X,Y)
\end{align*}

Let 
$
\mathcal{I}^t := \left\{
(q_1,q_2,q_3) : q_2+|q_3|=t, |q_1| \leq t \right\},
$
and 
$$
\mathcal{J}_{q_1, q_2, q_3} := \left\{
r = (r_1, r_2, r_{31}, \dots, r_{3K}, r_{41}, \dots, r_{4K}) : r_1 + \sum_{\ell=1}^{K} r_{3\ell} = q_1, |r_1|+r_2=q_2,  (r_{4\ell} + |r_{3\ell}|)_{\ell=1}^{K} = q_3 \right\}.
$$
$A_{n,1}$ can be expressed as
\begin{align*}
A_{n,1}  &= \sum_{j: |\gA_j|=1} \sum_{t=1}\sum_{(q_1,q_2,q_3) \in \mathcal{I}^t} \sum_{r \in \mathcal{J}_{q_1, q_2,q_3}}\sum_{i \in \gA_j} \frac{\exp(\gamma_i^n)}{r!} (\Delta \alpha_{ij}^n)^{r_1} (\Delta\beta_{ij}^n)^{r_2}\prod_{\ell=1}^{K} (\Delta a_{ij\ell}^n)^{r_{3\ell}} (\Delta b_{ij\ell}^n)^{r_{4\ell}} \\
&\quad \times X^{q_1} \times\frac{\partial^{q_2} \sigma((\alpha_j^*)^{\top}X + \beta_j^*)}{\partial {g}^{q_2}} \times\frac{\partial^{|q_3|} f}{\prod_{\ell=1}^{K} \partial h_{\ell}^{q_{31}}}(Y=s|X; a_j^*, b_j^*) + R_1(X,Y).
\end{align*}
Similarly, using a second-order Taylor expansion, $A_{n,2}$ can be written as
\begin{align*}
A_{n,2}  &= \sum_{j: |\gA_j|>1} \sum_{t=1}^2\sum_{(q_1,q_2,q_3) \in \mathcal{I}^t} \sum_{r \in \mathcal{J}_{q_1, q_2,q_3}}\sum_{i \in \gA_j} \frac{\exp(\gamma_i^n)}{r!} (\Delta \alpha_{ij}^n)^{r_1} (\Delta\beta_{ij}^n)^{r_2}\prod_{\ell=1}^{K} (\Delta a_{ij\ell}^n)^{r_{3\ell}} (\Delta b_{ij\ell}^n)^{r_{4\ell}} \\
&\quad \times X^{q_1} \times\frac{\partial^{q_2} \sigma((\alpha_j^*)^{\top}X + \beta_j^*)}{\partial {g}^{q_2}} \times\frac{\partial^{|q_3|} f}{\prod_{\ell=1}^{K} \partial h_{\ell}^{q_{31}}}(Y=s|X; a_j^*, b_j^*) + R_2(X,Y),
\end{align*}
Here $R_2(X, Y)$ denotes a Taylor remainder satisfying $R_2(X, Y) / \gD_1(G_n, \Gs) \to 0$ as $n \to \infty$. Applying analogous reasoning to decompose $B_n$, we obtain
\begin{align*}
B_n &= \sum_{j: |\gA_j|=1} \sum_{i \in \gA_j}\exp(\gamma_i^n)\sum_{|r|=1}\frac{1}{r!}(\Delta \alpha_i^n)^{r_1}(\Delta \beta_i^n)^{r_2}\frac{\partial^{|r|} \sigma((\alpha_j^*)^{\top}X + \beta_j^*)}{\partial \alpha_j^{r_1} \partial \beta_j^{r_2}} g_{\Gs}(Y=s|X) + R_3(X, Y) \\
&\quad + \sum_{j: |\gA_j|>1} \sum_{i \in \gA_j}\exp(\gamma_i^n)\sum_{|r|=1}^2\frac{1}{r!}(\Delta \alpha_i^n)^{r_1}(\Delta \beta_i^n)^{r_2}\frac{\partial^{|r|} \sigma((\alpha_j^*)^{\top}X + \beta_j^*)}{\partial \alpha_j^{r_1} \partial \beta_j^{r_2}} g_{\Gs}(Y=s|X) + R_4(X, Y) \\
&= \sum_{j: |\gA_j|=1} \sum_{i \in \gA_j}\exp(\gamma_i^n)\sum_{|r|=1}\frac{1}{r!}\times X^{r_1} \times (\Delta \alpha_i^n)^{r_1}(\Delta \beta_i^n)^{r_2}\frac{\partial^{|r|} \sigma((\alpha_j^*)^{\top}X + \beta_j^*)}{\partial g^{|r|}} g_{\Gs}(Y=s|X) + R_3(X, Y) \\
&\quad + \sum_{j: |\gA_j|>1} \sum_{i \in \gA_j}\exp(\gamma_i^n)\sum_{|r|=1}^2\frac{1}{r!}\times X^{r_1} \times(\Delta \alpha_i^n)^{r_1}(\Delta \beta_i^n)^{r_2}\frac{\partial^{|r|} \sigma((\alpha_j^*)^{\top}X + \beta_j^*)}{\partial g^{|r|}} g_{\Gs}(Y=s|X) + R_4(X, Y)
\end{align*}
Combining all the above results together, we obtain that
\begin{align*}
T_n(s) &= \sum_{j=1}^{k_*}\sum_{t=0}^{1+\mathbf{1}_{\{|\gA_j|>1\}}} \sum_{(q_1,q_2,q_3) \in \mathcal{I}^t} U_{q_1, q_2, q_3}^n(j) X^{q_1} \frac{\partial^{q_2} \sigma((\alpha_j^*)^{\top}X + \beta_j^*)}{\partial {g}^{q_2}} \times \frac{\partial^{|q_3|} f}{\prod_{\ell=1}^{K} \partial h_{\ell}^{q_{31}}}(Y=s|X; a_j^*, b_j^*) \\
&\quad + \sum_{j=1}^{k_*} \sum_{|r|=0}^{1+\mathbf{1}_{\{|\gA_j|>1\}}} W_r^n(j) \times X^{r_1} \times \frac{\partial^{|r|} \sigma((\alpha_j^*)^{\top}X + \beta_j^*)}{\partial g^{|r|}} \times p_{\Gs}(Y=s|X) + R(X, Y),
\end{align*}

where $R(X, Y)$ is the sum of Taylor remainders satisfying $R(X, Y) / \gD_1(X, Y) \to 0$ as $n \to \infty$ and
$$
U_{q_1, q_2,q_3}^n(j) = \begin{cases}
 \displaystyle\sum_{r \in \mathcal{J}_{q_1, q_2,q_3}}\sum_{i \in \gA_j} \frac{\exp(\gamma_i^n)}{r!} (\Delta \alpha_{ij}^n)^{r_1} (\Delta\beta_{ij}^n)^{r_2}\prod_{\ell=1}^{K} (\Delta a_{ij\ell}^n)^{r_{3\ell}} (\Delta b_{ij\ell}^n)^{r_{4\ell}}, & (q_1, q_2, q_3) \ne (\mathbf{0}_d, \mathbf{0}, \mathbf{0}_{K}) \\
    \sum_{i \in \gA_j} \exp(\gamma^n_{i}) - \exp(\gamma^*_{j}), & (q_1, q_2, q_3) = (\mathbf{0}_d, \mathbf{0},\mathbf{0}_{K})
\end{cases}
$$
and
$$
W_r^n(j) = \begin{cases}
    \sum_{i \in \gA_j}\frac{\exp(\gamma_i^n)}{r!}(\Delta \alpha_i^n)^{r_1}(\Delta \beta_i^n)^{r_2}, & |r| \ne \mathbf{0}_d, \\
    -\sum_{i \in \gA_j} \exp(\gamma^n_{i}) + \exp(\gamma^*_{j}), & |r| = \mathbf{0}_d,
\end{cases}
$$
for any $j \in [k_*]$.

\textbf{Step 2.} We now show that at least one among $U^n_{q_1,q_2,q_3}(j)/\gD_1(G_n, \Gs)$ does not vanish as $n \to \infty$. Assume, by contradiction, that all of them vanish as $n \to \infty$. 
By taking the summation of $|U^n_{q_1,q_2,q_3}(j)|/\gD_1(G_n, \Gs)$ for all $j \in \{j : |\gA_j|=1\}$, $q_2 = 1, q_3 = \mathbf{0}_K$, and $q_1 \in \{e_1, \ldots, e_d\}$, where $e_m := (0,\ldots,0,\underbrace{1}_{m\text{-th}},0,\ldots,0) \in \mathbb{R}^d$, we obtain
$$
\frac{1}{\gD_1(G_n, \Gs)} \sum_{j: |\gA_j|=1} \sum_{i \in \gA_j} \exp(\gamma_{i}^n) \|\Delta \alpha_{ij}^n\| \to 0.
$$
Similarly for $q_1 \in \{e_1, \ldots, e_d\}$, $q_2 = 0$, and iterating $q_3 \in \{e'_1, \ldots, e'_K\}$, where $e'_\ell := (0,\ldots,0,\underbrace{1}_{\ell\text{-th}},0,\ldots,0) \in \mathbb{R}^K$, we have
$$
\frac{1}{\gD_1(G_n, \Gs)} \sum_{j: |\gA_j|=1} \sum_{i \in \gA_j} \sum_{\ell=1}^K \exp(\gamma_{i}^n) \|\Delta a_{ij\ell}^n\| \to 0.
$$
On the other hand, taking $(q_1, q_2, q_3) = (\mathbf{0}_d, 1, \mathbf{0}_K)$ and considering the scalar component, we obtain that
$$
\frac{1}{\gD_1(G_n, \Gs)} \sum_{j: |\gA_j|=1} \sum_{i \in \gA_j} \exp(\gamma_{i}^n) |\Delta \beta_{ij}^n| \to 0.
$$
Similarity for $(q_1, q_2, q_3) = (\mathbf{0}_d, 0, e'_\ell)$ for each $\ell \in [K]$, we obtain
$$
\frac{1}{\gD_1(G_n, \Gs)} \sum_{j: |\gA_j|=1} \sum_{i \in \gA_j} \sum_{\ell=1}^K \exp(\gamma_{i}^n) |\Delta b_{ij\ell}^n| \to 0.
$$
Combining the 4 results above, we deduce that
\begin{equation}\label{eq: limit1_style1}
    \frac{1}{\gD_1(G_n, \Gs)} \sum_{j:|\gA_j|=1}\sum_{i \in \gA_j} \exp(\gamma_{i}^n) \left[ \|\Delta \alpha_{ij}^n\| + |\Delta \beta_{ij}^n| + \sum_{\ell=1}^K (\|\Delta a_{ij\ell}^n\| + |\Delta b_{ij\ell}^n|) \right] \to 0.
\end{equation}
By applying the same argument for the indices where $|\gA_j| > 1$ , we also deduce that
\begin{equation}\label{eq: limit2_style1}
    \frac{1}{\gD_1(G_n, \Gs)} \sum_{j:|\gA_j|>1}\sum_{i \in \gA_j} \exp(\gamma_{i}^n) \left[ \|\Delta \alpha_{ij}^n\|^2 + |\Delta \beta_{ij}^n|^2 + \sum_{\ell=1}^K (\|\Delta a_{ij\ell}^n\|^2 + |\Delta b_{ij\ell}^n|^2) \right] \to 0.
\end{equation}
Additionally, for tuple $(q_1,q_2,q_3) = (\mathbf{0}_d, 0, \mathbf{0}_K)$, it follows that
\begin{equation}\label{eq: limit3_style1}
    \frac{1}{\gD_1(G_n, \Gs)} \cdot \left(\sum_{j=1}^{k_*} \left|\sum_{i\in\gA_j} \exp(\gamma_{i}^n) - \exp(\gamma_{j}^*)\right|\right) \to 0.
\end{equation}
It is induced from the limits in equations \eqref{eq: limit1_style1}, \eqref{eq: limit2_style1}, \eqref{eq: limit3_style1} that $1 = \gD_1(G_n, \Gs)/\gD_1(G_n, \Gs) \to 0$ when $n \to \infty$, which is a contradiction. Thus, at least one among $U^n_{q_1,q_2,q_3}(j)/\gD_1(G_n, \Gs)$ does not go to zero when $n \to \infty$.

\noindent\textbf{Step 3.} Finally, we invoke Fatou's lemma to achieve a contradiction to the result in Step 2.

Let $m_n$ denote the maximum of the absolute values of 
$U_{q_1, q_2,q_3}^n(j) / \gD_1(G_n, \Gs)$ and 
$W_r^n(j) / \gD_1(G_n, \Gs)$. Then, by Fatou's lemma, 
\[
0 = \lim_{n \to \infty} \frac{\mathbb{E}_X\big[ 2d_V(p_{G_n}(\cdot|X), p_{G^*}(\cdot|X)) \big]}{m_n \gD_1(G_n, \Gs)}
\ge  \int \sum_{s=1}^K \liminf_{n \to \infty} \left| \frac{p_{G_n}(Y=s|X) - p_{\Gs}(Y=s|X)}{m_n \gD_1(G_n, \Gs)} \right| \mathrm{d} X \ge 0.
\]

As a result, we get
\[
\frac{|p_{G_n}(Y=s|X) - p_{\Gs}(Y=s|X)|}{m_n \gD_1(G_n, \Gs)} \to 0,
\]
which implies that
\[
\frac{T_n(s)}{m_n \gD_1(G_n, \Gs)} \to 0
\quad \text{as } n \to \infty, \quad \forall s \in [K], \text{ a.s. } X.
\]

Let 
\[
\frac{U_{q_1, q_2,q_3}^n(j)}{m_n \gD_1(G_n, \Gs)} \to \tau_{q_1, q_2,q_3}(j), 
\quad
\frac{W_r^n(j)}{m_n \gD_1(G_n, \Gs)} \to \eta_r(j)
\quad \text{as } n \to \infty.
\]
Then, the previous result indicates that
\begin{align*}
T(s) &= \sum_{j=1}^{k_*}\sum_{t=0}^{1+\mathbf{1}_{\{|\gA_j|>1\}}} \sum_{(q_1,q_2,q_3) \in \mathcal{I}^t} \tau_{q_1, q_2,q_3}(j) \, X^{q_1} \,
\frac{\partial^{q_2} \sigma((\alpha_j^*)^{\top}X + \beta_j^*)}{\partial g^{q_2}} \times
\frac{\partial^{|q_3|} f}{\prod_{\ell=1}^{K} \partial h_\ell^{q_{3\ell}}}(Y=s|X; a_j^*, b_j^*) \\
&\quad + \sum_{j=1}^{k_*} \sum_{|r|=0}^{1+\mathbf{1}_{\{|\gA_j|>1\}}} \eta_r(j) \, X^{r_1} \,
\frac{\partial^{|r|} \sigma((\alpha_j^*)^{\top}X + \beta_j^*)}{\partial g^{|r|}} \,
p_{\Gs}(Y=s|X) = 0,
\end{align*}
for any $s \in [K]$ and almost surely $X$. 
However, this implies that all the coefficients $\tau_{q_1,q_2,q_3}(j)$ and $\eta_r(j)$ must be zero, which is a contradiction.

\noindent\textbf{Global Structure:} Because the local inequality holds, there exists a constant $\epsilon' > 0$ such that
\[
\inf_{\substack{
G \in \mathcal{G}_k(\Theta) \\
\gD_1(G, G^*) \le \epsilon'
}}
\left(
\mathbb{E}_X\!\left[
d_V\!\bigl(
p_G(\cdot \mid X),\,
p_{G^*}(\cdot \mid X)
\bigr)
\right]
\,/\, \gD_1(G, G^*)
\right)
> 0.
\]
Assume, by contradiction, that this is not true. Then there exists a sequence $G_n' \in \mathcal{G}_k(\Theta)$ such that $\gD_1(G_n', \Gs) > \epsilon'$ and 
\[
\mathbb{E}_X[d_V(p_{G'_n}(\cdot|X), p_{G^*}(\cdot|X))] \to 0 \quad \text{as } n \to \infty.
\]
Since $\Theta$ is compact, we can extract a subsequence such that $G_n' \to G' \in \mathcal{G}_k(\Theta)$. 
Thus, $\gD_1(G', \Gs) > \epsilon'$. Applying Fatou's lemma yields
\[
0 = \lim_{n \to \infty} \frac{\mathbb{E}_X[2d_V(p_{G'_n}(\cdot|X), p_{G^*}(\cdot|X))]}{\gD_1(G_n', \Gs)}
\ge \int \sum_{s=1}^K \liminf_{n \to \infty} |p_{G_n'}(Y=s|X) - p_{\Gs}(Y=s|X)| \, \mathrm{d}X.
\]
Consequently,
\[
\int \sum_{s=1}^K |p_{G'}(Y=s|X) - p_{\Gs}(Y=s|X)| \, \mathrm{d}X = 0,
\]
which implies $p_{G'}(Y=s|X) = p_{\Gs}(Y=s|X)$ for all $s \in [K]$ almost surely in $X$. 
By identifiability, we conclude $G' = \Gs$, which contradicts $\gD_1(G', \Gs) > \epsilon'$. 

Hence, the global inequality holds, completing the proof.

\subsection{Proof of Theorem~\ref{theorem:with_temperature_inner}}
\label{appendix:with_temperature_inner}

First of all, we provide a useful lemma that will be utilized for this proof as follows:

\begin{lemma}\label{lemma A.1}
    For any $r \ge 1$, if the following holds:
$$
\lim_{\varepsilon \to 0}\inf_{G \in \gG_k(\Upsilon): \gD_{2,r}(G, \widetilde{G}^*)\le \varepsilon}\dfrac{\mathbb{E}_X[d_V(\widetilde{p}_G(\cdot|X), \widetilde{p}_{\widetilde{G}^*}(\cdot |X))]}{\gD_{2,r}(G, \widetilde{G}^*)} = 0,
$$
then we achieve that 
$$\inf_{G_n \in \gG_k(\Upsilon)}\sup_{G \in \gG_k(\Upsilon)\setminus \gG_{\ks-1}(\Upsilon)}\mathbb{E}_{\widetilde{p}_G}[\gD_{2,r}(G_n,G)] \gtrsim n^{-1/2}.$$
\end{lemma} 
Proof of Lemma~\ref{lemma A.1} is deferred to Appendix \ref{appendix: proof of lemma 3.1}. Now we are ready to present the main proof of Theorem~\ref{theorem:with_temperature_inner}.  
\subsubsection{Main Proof}
Based on the result of Lemma~\ref{lemma A.1}, it is sufficient to construct a sequence of mixing measures $G_n$ such that $\gD_{2,r}(G_n, \widetilde{G}^*) \to 0$ and 
\begin{equation}\label{eq: frac to 0}
\dfrac{\mathbb{E}_X[d_V(\widetilde{p}_{G_n}(\cdot|X), \widetilde{p}_{\widetilde{G}^*}(\cdot |X))]}{\gD_{2,r}(G, \widetilde{G}^*)} \to 0,  
\end{equation}
as $n\to \infty$. For that purpose, we choose $G_n = \sum_{i=1}^{k^*+1}\exp(\gamma_{i}^n)\delta_{(\alpha_i^n, \beta_i^n, \tau^n, a_i^n, b_i^n)}$ such that 
\begin{itemize}
    \item $a_{i}^n = a_{j}^*, b_i^n = b_j^*$ for any $i \in \gA_j$;
    \item $\tau^n = \tau^* + t_n$;
    \item $\beta_{i}^n = \beta_{j}^* + s_{n,j}$ for any $i \in \gA_j$;
    \item $a_{i}^n = a_j^* + v_{n,j}$ for any $i \in \gA_j$;
    \item $\exp(\gamma_i^n) = \frac{1}{|\gA_j|}\exp(\gamma_j^*)$ for any $i \in \gA_j$, 
\end{itemize}
where $v_{n,j} := (v_{n,j}^{(1)},\ldots,v_{n,j}^{(d)}) \in \mathbb{R}^d$ and $t_n, s_{n,j} \in \mathbb{R}$ will be chosen later such that $v_{n,j}^{(u)},s_{n,j},t_n \to 0$ when $n \to \infty$ for any $u \in [d]$ and $j \in [\ks]$. Then, the loss function $\gD_{2,r}(G_n, \widetilde{G}^*)$ becomes
$$
\gD_{2,r}(G_n, \widetilde{G}^*) = \sum_{j=1}^{\ks}|\gA_j|\cdot \left(\|v_{n,j}\|^r + |t_n|^r + |s_{n,j}|^r\right)
$$
It is straightforward that $\gD_{2,r}(G_n, \widetilde{G}^*)\to 0$ when $n\to \infty$. 

Now, we will show that $\mathbb{E}_X[d_H(\widetilde{p}_{G_n}(\cdot|X), \widetilde{p}_{\widetilde{G}^*}(\cdot |X))]/\gD_{2,r}(G_n, \widetilde{G}^*)$ vanishes as $n \to \infty$. Let us consider the quantity $T_n(s) = \left[\sum_{i=1}^{\ks}\exp(\gamma_i^*)\sigma\left(\frac{(\alpha_i^*)^\top X+\beta_i^*}{\tau^*}\right)\right][\widetilde{p}_{G_n}(Y=s|X) - \widetilde{p}_{\widetilde{G}^*}(Y=s|X)]$ under the above setting of $G_n$ as follows
\begin{align*}
    T_n(s) &:= \sum_{j=1}^{k^*}\sum_{i \in \gA_j}\exp(\gamma_i^n)[\widetilde{u}(Y=s|X;\alpha_i^n, \beta_i^n,\tau^n, a_i^n, b_i^n) - \widetilde{u}(Y=s|X;\alpha_j^*, \beta_j^*,\tau^*, a_j^*, b_j^*)] \\
    &- \sum_{j=1}^{k^*}\sum_{i \in \gA_j}\exp(\gamma_i^n)[\widetilde{v}(Y=s|X;\alpha_i^n, \beta_i^n,\tau^n) - \widetilde{v}(Y=s|X;\alpha_j^*,\beta_j^*,\tau^*)] \\
    &+\sum_{j=1}^{k^*}\Big[\sum_{i \in \gA_j}\exp(\gamma_i^n) - \exp(\gamma_j^*)\Big]\Big[\widetilde{u}(Y=s|X;\alpha_j^*,\beta_j^*,\tau^*, a_j^*, b_j^*) - \widetilde{u}(Y=s|X;\alpha_j^*, \beta_j^*,\tau^*)\Big] \\
    &:= A_n - B_n + E_n,
\end{align*}
where we denote  $f(Y=s|X;a,b) := \dfrac{\exp(a_s^\top X + b_s)}{\sum_{\ell=1}^K\exp(a_\ell^\top X + b_\ell)}$, $\widetilde{u}(Y=s|X;\alpha,\beta,\tau,a,b) := \sigma\Big(\frac{\alpha^\top X + \beta}{\tau}\Big)f(Y=s|X;a,b)$ and $\widetilde{v}(Y=s|X;\alpha,\beta,\tau):= \sigma\Big(\frac{\alpha^\top X + \beta}{\tau}\Big) p_{G_n}(Y=s|X)$ for each $s \in [K]$. Given the above formulation of $G_n$, we can rewrite the term $A_n$ as 
$$
A_n = \sum_{j=1}^{k^*}\sum_{i \in \gA_j}\exp(\gamma_i^n) \Bigg[\sigma\left(\frac{(\alpha_i^n)^\top X+\beta_i^n}{\tau^n}\right) - \sigma\left(\frac{(\alpha_j^*)^\top X+\beta_j^*}{\tau^*}\right)\Bigg]f(Y=s|X;a^*,b^*)
$$
By means of first-order Taylor expansions,  we can rewrite $A_n$ as 
\begin{align*}
   A_n = \sum_{j=1}^{k^*}\sum_{i \in \gA_j}\exp(\gamma_i^n)\Bigg\{\sum_{u=1}^d\Big[\dfrac{v_{n,j}^{(u)}}{\tau^*}-\dfrac{t_n(\alpha_{j}^*)^{(u)})}{(\tau^*)^2}\Big]\cdot X^{(u)}+\Big[\dfrac{s_{n,j}}{\tau^*}-\dfrac{t_n\beta_j^*}{(\tau^*)^2}\Big]\Bigg\}\sigma\left(\frac{(\alpha_j^*)^\top X+\beta_j^*}{\tau^*}\right)\\
   \times \Bigg[1-\sigma\left(\frac{(\alpha_j^*)^\top X+\beta_j^*}{\tau^*}\right)\Bigg]f(Y=s|X;a^*,b^*) + R_1(X,Y),
\end{align*}
where $R_1(X,Y)$ is a Taylor remainder such that $R_1(X,Y)/\gD_{2,r}(G_n, \widetilde{G}^*) \to 0$ when $n\to \infty$. Hence, by choosing 
$$
t_n = \dfrac{1}{n};\qquad v_{n,j} = \dfrac{t_n\alpha_j^*}{\tau^*} = \dfrac{\alpha_j^*}{n\tau^*};\qquad s_{n,j} = \dfrac{t_n\beta_j^*}{\tau^*} = \dfrac{\beta_j^*}{n\tau^*},
$$
we obtain that $A_n/\gD_{2,r}(G_n, \widetilde{G}^*) \to 0$ when $n\to\infty$. 

By following the same line of reasoning, we also obtain that $B_n/\gD_{2,r}(G_n, \widetilde{G}^*) \to 0$ when $n\to \infty$. Combining with $E_n = 0$, it follows that $T_n(s) / \gD_{2,r}(G_n, \widetilde{G}^*) \to 0$ when $n\to \infty$. Since the term $\sum_{i=1}^{\ks}\exp(\gamma_i^*)\sigma\left(\frac{(\alpha_i^*)^\top X+\beta_i^*}{\tau^*}\right)$ is bounded, we deduce that $|\widetilde{p}_{G_n}(Y|X) - \widetilde{p}_{\widetilde{G}^*}(Y|X)|/\gD_{2,r}(G_n, \widetilde{G}^*)\to 0$, for almost surely $(X,Y)$ when $n\to \infty$. As a consequence, we satisfy the condition in equation \eqref{eq: frac to 0}. This completes the proof. 

\subsubsection{Proof of Lemma A.1.}\label{appendix: proof of lemma 3.1}
For a fixed constant $M_1 > 0$ and a sufficiently small $\varepsilon > 0$ that we will choose later, it follows from the assumption that we can find a mixing measure $G^*_1 \in \gG_{k}(\Upsilon)$ that satisfies $\gD_{2,r}(G^*_1, \widetilde{G}^*) = 2\varepsilon$ and $\mathbb{E}_X[d_V(\widetilde{p}_{G^*_1}(\cdot|X), \widetilde{p}_{\widetilde{G}^*}(\cdot |X))] \le M_1\varepsilon$. Additionally, for any sequence $G_n' \in \gG_k(\Upsilon)$, we have 
$$
2\max_{G \in \{G^*_1, \widetilde{G}^*\}} \mathbb{E}_{\widetilde{p}_G}[\gD_{2,r}(G'_n, G)] \ge \mathbb{E}_{\widetilde{p}_{\widetilde{G}^*}}[\gD_{2,r}(G'_n, \widetilde{G}^*)] + \mathbb{E}_{\widetilde{p}_{G^*_1}}[\gD_{2,r}(G'_n, G^*_1)],
$$
where $\mathbb{E}_{\widetilde{p}_G}$ denotes the expectation taken w.r.t the product measure with density $\widetilde{p}_G$. Since the loss $\gD_{2,r}$ satisfies the weak triangle inequality, we can find a constant $M_2 > 0$ such that 
$$
\gD_{2,r}(G'_n, \widetilde{G}^*) + \gD_{2,r}(G'_n, G^*_1) \ge M_2\gD_{2,r}(G^*_1, \widetilde{G}^*) = 2M_2\varepsilon.
$$
Consequently, it follows that 
\begin{align*}
    \max_{G \in \{G^*_1, \widetilde{G}^*\}} \mathbb{E}_{\widetilde{p}_G}[\gD_{2,r}(G'_n, G)] &\ge \dfrac{1}{2}\left(\mathbb{E}_{\widetilde{p}_{\widetilde{G}^*}}[\gD_{2,r}(G'_n, \widetilde{G}^*)] + \mathbb{E}_{\widetilde{p}_{G^*_1}}[\gD_{2,r}(G'_n, G^*_1)]\right)\nonumber\\
    &\ge M_2\varepsilon 
    \cdot \inf_{f_1,f_2} \left(\mathbb{E}_{\widetilde{p}_{\widetilde{G}^*}}[f_1] + \mathbb{E}_{\widetilde{p}_{G^*_1}}[f_2]\right),
\end{align*}
where $f_1$ and $f_2$ in the above infimum are measurable functions in terms of $X_1, X_2, \ldots, X_n$ that satisfy $f_1 + f_2 = 1$.

With the definition of the Total Variation distance, the infimum value $\inf_{f_1,f_2} \left(\mathbb{E}_{\widetilde{p}_{\widetilde{G}^*}}[f_1] + \mathbb{E}_{\widetilde{p}_{G^*_1}}[f_2]\right)$ is equal to $1 - d_V(\widetilde{p}_{\widetilde{G}^*}(\cdot|X), \widetilde{p}_{G^*_1}(\cdot |X))$. Thus, we obtain that 
\begin{align*}
    \max_{G \in \{G'_*, \widetilde{G}^*\}} \mathbb{E}_{p_G}[\gD_{2,r}(G_n', G)] &\ge M_2\varepsilon \Big(1 - d_V(p_{\widetilde{G}^*}(\cdot|X), p_{G^*_1}(\cdot |X))\Big) \nonumber\\
    &\ge M_2\varepsilon\left[1 - \sqrt{1 - (1-M_1^2\varepsilon^2)^n}\right].
\end{align*}
By choosing $\varepsilon = n^{-1/2}/M_1$, we have $M_1^2\varepsilon^2 = \frac{1}{n}$ and $\sqrt{1 - (1-M_1^2\varepsilon^2)^n} = \sqrt{1-(1-1/n)^n} \to \sqrt{(1 - e^{-1})} < 1$, which implies that 
$$
\inf_{G_n' \in \gG_k(\Upsilon)}\sup_{G \in \gG_k(\Upsilon)\setminus \gG_{\ks-1}(\Upsilon)}\mathbb{E}_{\widetilde{p}_G}[\gD_{2,r}(G_n',G)] \ge \max_{G \in \{G^*_1, \widetilde{G}^*\}} \mathbb{E}_{\widetilde{p}_G}[\gD_{2,r}(G'_n, G)] \gtrsim n^{-1/2}.
$$
Hence, we reach the conclusion of Theorem~\ref{theorem:with_temperature_inner}, which says that 
$$
\inf_{G_n \in \gG_k(\Upsilon)}\sup_{G \in \gG_k(\Upsilon)\setminus \gG_{\ks-1}(\Upsilon)}\mathbb{E}_{\widetilde{p}_G}[\gD_{2,r}(G_n,G)] \gtrsim n^{-1/2},
$$
for any $r \ge 1$.

\subsection{Proof of Theorem~\ref{theorem:with_temperature_euclidean}}
\label{appendix:with_temperature_euclidean}
To reach the desired conclusion in Theorem~\ref{appendix:with_temperature_euclidean}, we need to show the following key inequality:
\begin{equation}
    \inf_{G \in \gG_k(\Upsilon)}\mathbb{E}_X[d_V(\overline{p}_{G}(\cdot|X), \overline{p}_{\overline{G}^*}(\cdot |X))] / \gD_3(G,\overline{G}^*) > 0,
\end{equation}
which is then divided into two parts named local structure and global structure. Since the global structure can be treated analogously to the case in Appendix~\ref{appendix:without_temperature}, its proof is omitted here for brevity.

\textbf{Local Structure.} In this part, we aim to establish the following inequality using a proof by contradiction:
$$
\lim_{\varepsilon \to 0} \inf_{\substack{G \in \gG_k(\Upsilon)},\\\gD_{3}(G, \overline{G}^*)\le \varepsilon} \mathbb{E}_X[d_V(\overline{p}_{G}(\cdot|X), \overline{p}_{\overline{G}^*}(\cdot |X))]/ \gD_3(G,\overline{G}^*) > 0.
$$
Assume that this local inequality does not hold. By utilizing some derivations in Appendix~\ref{appendix:without_temperature}, we proceed with the three-step framework as follows:

\textbf{Step 1.} First of all, we define 
\begin{align*}
    \overline{u}(Y=s|X;\alpha,\beta, \tau,a,b) &:= \sigma\left(\frac{\|\alpha-X\|+\beta}{\tau}\right)f(Y=s|X;a,b);\\
    \overline{v}(Y=s|X;a,b,\tau) &:= \sigma\left(\frac{\|\alpha-X\|+\beta}{\tau}\right)\overline{p}_{\overline{G}_n}(Y=s|X),
\end{align*}
for any $s \in [K]$. Next,  decompose the following difference as follows
\begin{align*}
T_n(s)&:=\Big[\sum_{j=1}^{\ks}\exp(\gamma_j^*)\sigma\left(\frac{\|\alpha_j^*-X\| + \beta_j^*}{\tau^*}\right)\Big] (\overline{p}_{\overline{G}_n}(Y=s|X) - \overline{p}_{\overline{G}^*}(Y=s|X)) \\
&= \sum_{j=1}^{\ks}\sum_{i \in \gA_j} \exp(\gamma_i^n)[\overline{u}(Y=s|X;\alpha_j^n, \beta_j^n, \tau^n, a^n, b^n) - \overline{u}(Y=s|X;\alpha_j^*,\beta_j^*, \tau^*, a^*, b^*)] \\
& \quad - \sum_{j=1}^{\ks}\sum_{i \in \gA_j} \exp(\gamma_i^n)[\overline{v}(Y=s|X;\alpha_j^n, \beta_j^n, \tau^n) - \overline{v}(Y=s|X;\alpha_j^*,\beta_j^*,\tau^* )] \\
& \quad + \sum_{j=1}^{\ks}\left(\sum_{i \in \gA_j}\exp(\gamma_i^n) - \exp(\gamma_j^*)\right) \overline{u}(Y=s|X;\alpha_j^*,\beta_j^*,\tau^*,a^*,b^*) \\
&\quad - \sum_{j=1}^{\ks}\left(\sum_{i \in \gA_j}\exp(\gamma_i^n) - \exp(\gamma_j^*)\right)\overline{v}(Y=s|X;\alpha_j^*,\beta_j^*,\tau^*)\\
&:= A_n-B_n+E_{n,1}-E_{n,2}.
\end{align*}
Given the above formulations of $A_n$ and $B_n$, we further decompose each into two smaller components according to the cardinality of the Voronoi cell $\gA_j$. In particular,
\begin{align*}
    A_n &= \sum_{j: |\gA_j| = 1}\sum_{i \in \gA_j} \exp(\gamma_i^n)[\overline{u}(Y=s|X;\alpha_j^n, \beta_j^n, \tau^n, a^n, b^n) - \overline{u}(Y=s|X;\alpha_j^*,\beta_j^*, \tau^*, a^*, b^*)] \\ 
    &\quad + \sum_{j: |\gA_j| > 1}\sum_{i \in \gA_j} \exp(\gamma_i^n)[\overline{u}(Y=s|X;\alpha_j^n, \beta_j^n, \tau^n, a^n, b^n) - \overline{u}(Y=s|X;\alpha_j^*,\beta_j^*, \tau^*, a^*, b^*)] \\
    &:= A_{n,1} + A_{n,2}.
\end{align*}
Let us denote $h_{\ell}(X; a_{i\ell}, b_{i\ell}):= (a_{i\ell})^\top X+ b_{i\ell}, \quad g_i(X; \alpha_i, \beta_i,\tau) := (\|\alpha_i - X \|+ \beta_i)/\tau$  for any $\ell \in [K]$ and $i \in [k]$. By means of the first-order Taylor expansion, $A_{n,1}$ can be represented as 
\begin{align*}
    A_{n,1} &= \sum_{j:|\gA_j| =1}\sum_{i \in \gA_j} \exp(\gamma_i^n)\sum_{|t|=1}\dfrac{1}{t!}(\Delta\alpha_{ij})^{t_1}(\Delta \beta_{ij})^{t_2}(\Delta\tau_{ij})^{t_3}\prod_{\ell=1}^K(\Delta a_{ij\ell}^n)^{t_{4\ell}}(\Delta b_{ij\ell}^{n})^{t_{5\ell}} \\
    &\quad \times \dfrac{\partial^{|t_1|+t_2+t_3}\sigma \circ g_j}{\partial \alpha_j^{t_1}\partial\beta_j^{t_2}\partial\tau^{t_3}}(X;\alpha_j^*, \beta_j^*, \tau^*)\times \dfrac{\partial^{\sum_{\ell=1}^K(|t_{4\ell}|+t_{5\ell})}f}{\prod_{\ell=1}^K\partial a_{j\ell}^{|t_{4\ell}|}\partial b_{j\ell}^{t_{5\ell}}}(Y=s|X;a_j^*, b_j^*) + R_1(X,Y),
\end{align*}
where $t=( t_1,t_2,t_3, t_{41}, \ldots,t_{4K}, t_{51}, \ldots, t_{5K})$ with $t_1, t_{4\ell} \in \mathbb{N}^d$, $t_2,t_3, t_{5\ell} \in \mathbb{N}$ for all $\ell \in [K]$. Additionally, $R_1(X,Y)$ is a Taylor remainder such that $R_1(X,Y)/D_{3}(\overline{G}_n,\overline{G}^*) \to 0$ as $n \to \infty$. From the formulation of function $h$, we have 
\begin{align*}
    A_{n,1} &= \sum_{j: |\gA_j| = 1}\sum_{i \in \gA_j} \exp(\gamma_i^n)\sum_{|t|=1}\dfrac{1}{t!}(\Delta\alpha_{ij})^{t_1}(\Delta \beta_{ij})^{t_2}(\Delta\tau_{ij})^{t_3}\prod_{\ell=1}^K(\Delta a_{ij\ell}^n)^{t_{4\ell}}(\Delta b_{ij\ell}^{n})^{t_{5\ell}} \dfrac{\partial^{|t_1|+t_2+t_3}g_j}{\partial \alpha_j^{t_1}\beta_j^{t_2}\tau^{t_3}}(X;\alpha_j^*,\beta_j^*,\tau_j^*) \\
    &\quad \times X^{\sum_{\ell=1}^Kt_{4\ell}} \times \dfrac{\partial^{|t_1|+t_2+t_3}\sigma \circ g_j}{\partial g_j^{|t_1|+t_2+t_3}}(X;\alpha_j^*,\beta_j^*, \tau_j^*)\times \dfrac{\partial^{\sum_{\ell=1}^K(|t_{4\ell}|+t_{5\ell})}f}{\prod_{\ell=1}^K\partial h_{\ell}^{|t_{4\ell}|+t_{5\ell}}}(Y=s|X;a_j^*, b_j^*) + R_1(X,Y).
\end{align*}
We define $q_1 = \sum_{\ell=1}^Kt_{4\ell} \in \mathbb{N}^d, q_2 = |t_1| + t_2+t_3 \in \mathbb{N}$, $q_3 = (q_{3\ell})_{\ell=1}^K:= (|t_{4\ell}|+ t_{5\ell})_{\ell=1}^K \in \mathbb{N}^K$. Moreover, we denote 
$\gI^z := \{(q_1,q_2,q_3) \in \mathbb{N}^d \times \mathbb{N}\times \mathbb{N}^K : q_2 + |q_3|= z, |q_1| \le |q_3|\}$ and $$\gJ_{q_1,q_2,q_3}:= \left\{t = ( t_1,t_2,t_3, t_{41}, \ldots,t_{4K}, t_{51}, \ldots, t_{5K}): \sum_{\ell=1}^Kt_{4\ell} = q_1, |t_1|+t_2+t_3 = q_2, (|t_{4\ell}|+t_{5\ell})_{\ell=1}^K=q_3\right\}.$$
The term $A_{n,1}$ can be written as follows:
\begin{align*}
    A_{n,1} &= \sum_{j:|\gA_j|=1}\sum_{t=1}\sum_{(q_1,q_2,q_3) \in \gI^t}\sum_{t \in \gJ_{q_1,q_2,q_3}}\sum_{i \in \gA_j} \dfrac{\exp(\gamma_i^n)}{t!}(\Delta\alpha_{ij})^{t_1}(\Delta \beta_{ij})^{t_2}(\Delta\tau_{ij})^{t_3}\prod_{\ell=1}^K(\Delta a_{ij\ell}^n)^{t_{4\ell}}(\Delta b_{ij\ell}^{n})^{t_{5\ell}} \\
    &\quad \times \dfrac{\partial^{|t_1|+t_2+t_3}g_j}{\partial \alpha_j^{t_1}\partial\beta_j^{t_2}\partial\tau^{t_3}}(X;\alpha_j^*,\beta_j^*,\tau_j^*) \times X^{q_1} \times \sigma^{(q_2)}(g_j (X;\alpha_j^*,\beta_j^*, \tau_j^*)) \times \dfrac{\partial^{|q_3|}f}{\partial h^{q_3}}(Y=s|X;a_j^*, b_j^*)\\
    &+ R_1(X,Y).
\end{align*}
Similarly, by means of second-order Taylor expansion, the term $A_{n,2}$ can be expressed as follows
\begin{align*}
    A_{n,2} &= \sum_{j:|\gA_j|>1}\sum_{t=1}^2\sum_{(q_1,q_2,q_3) \in \gI^t}\sum_{t \in \gJ_{q_1,q_2,q_3}}\sum_{i \in \gA_j} \dfrac{\exp(\gamma_i^n)}{t!}(\Delta\alpha_{ij})^{t_1}(\Delta \beta_{ij})^{t_2}(\Delta\tau_{ij})^{t_3}\prod_{\ell=1}^K(\Delta a_{ij\ell}^n)^{t_{4\ell}}(\Delta b_{ij\ell}^{n})^{t_{5\ell}} \\
    &\quad \times \dfrac{\partial^{|t_1|+t_2+t_3}g_j}{\partial \alpha_j^{t_1}\partial\beta_j^{t_2}\partial\tau^{t_3}}(X;\alpha_j^*,\beta_j^*,\tau_j^*) \times X^{q_1} \times \sigma^{(q_2)}(g_j (X;\alpha_j^*,\beta_j^*, \tau_j^*)) \times \dfrac{\partial^{|q_3|}f}{\partial h^{q_3}}(Y=s|X;a_j^*, b_j^*) \\&+ R_2(X,Y), 
\end{align*}
where $R_2(X,Y)$ is the Taylor remainder such that $R_2(X,Y)/\gD_3(\overline{G}_n, \overline{G}^*) \to 0$ as $n\to \infty$. 

Similarly, by applying the same arguments for decomposing $A_n$, we can express $B_n = B_{n,1} + B_{n,2}$ where 
\begin{align*}
    B_{n,1} &= \sum_{j:|\gA_j|=1}\sum_{i \in \gA_j}\exp(\gamma_j^n)\sum_{|t|=1}\dfrac{1}{t!}(\Delta \alpha_{ij}^n)^{t_1}(\Delta\beta_{ij}^n)^{t_2}(\Delta\tau_{ij})^{t_3}\sigma^{(|t|)}(g_j(X;\alpha_j^*, \beta_j^*, \tau^*))\\
    &\qquad \times \dfrac{\partial^{|t|}g_j}{\partial\alpha_j^{t_1}\partial\beta_j^{t_2}\partial\tau^{t_3}}(X;\alpha_j^*, \beta_j^*, \tau^*) \overline{p}_{\overline{G}_n}(Y=s|X)+R_3(X,Y) \\
    B_{n,2} &= \sum_{j:|\gA_j|>1}\sum_{i \in \gA_j}\exp(\gamma_j^n)\sum_{|t|=1}^2\dfrac{1}{t!}(\Delta \alpha_{ij}^n)^{t_1}(\Delta\beta_{ij}^n)^{t_2}(\Delta\tau_{ij})^{t_3}\sigma^{(|t|)}(g_j(X;\alpha_j^*, \beta_j^*, \tau^*))\\
    &\qquad \times \dfrac{\partial^{|t_1| + t_2 + t_3}g_j}{\partial\alpha_j^{t_1}\partial\beta_j^{t_2}\partial\tau^{t_3}}(X;\alpha_j^*, \beta_j^*, \tau^*) \overline{p}_{\overline{G}_n}(Y=s|X)+R_4(X,Y),
\end{align*}
where $R_3(X,Y)$ and $R_4(X,Y)$ are Taylor remainders such that 

$R_3(X,Y)/\gD_3(\overline{G}_n, \overline{G}^*) \to 0, R_4(X,Y)/\gD_3(\overline{G}_n, \overline{G}^*) \to 0$ as $n \to \infty$. We can verify that 
\begin{align}
    A_n + E_{n,1} &= \sum_{j=1}^{\ks}\sum_{t=0}^{1+\mathbf{1}_{\{|\gA_j|>1\}}}\sum_{(q_1,q_2,q_3 )\in \gI^{t}}U^n_{q_1,q_2,q_3}(j)\times \dfrac{\partial^{|t_1|+t_2+t_3}g_j}{\partial \alpha_j^{t_1}\partial\beta_j^{t_2}\partial\tau^{t_3}}(X;\alpha_j^*,\beta_j^*,\tau_j^*)\nonumber\\
    &\quad \times X^{q_1} \times \sigma^{(q_2)}(g_j (X,\alpha_j^*,\beta_j^*, \tau_j^*)) \times \dfrac{\partial^{|q_3|}f}{\partial h^{q_3}}(Y=s|X;a_j^*, b_j^*)+R_1(X,Y) + R_2(X,Y); \label{eq: a_n + e_n,1}\\
    B_n+E_{n,2} &= \sum_{j=1}^{\ks}\sum_{t=0}^{1 + \mathbf{1}_{|\gA_j|>1}} W_t^n(j) \dfrac{\partial^{|t_1|+t_2+t_3}g_j}{\partial\alpha_j^{t_1}\partial\beta_j^{t_2}\partial\tau^{t_3}}(X;\alpha_j^*, \beta_j^*, \tau^*) \nonumber \\
    &\quad \times \sigma^{(|t|)}(g_j(X;\alpha_j^*, \beta_j^*, \tau^*)) \overline{p}_{\overline{G}_n}(Y=s|X)+R_3(X,Y) + R_4(X,Y) \label{eq: b_n + e_n,2}
\end{align}
where we denote $U^n_{q_1,q_2,q_3}(j)$ equals to 
$$
\begin{cases}
    \displaystyle\sum_{i \in \gA_j}\sum_{t \in \gJ_{q_1,q_2,q_3}} \dfrac{\exp(\gamma_i^n)}{t!}(\Delta\alpha_{ij})^{t_1}(\Delta \beta_{ij})^{t_2}(\Delta\tau_{ij})^{t_3}\prod_{\ell=1}^K(\Delta a_{ij\ell}^n)^{t_{4\ell}}(\Delta b_{ij\ell}^{n})^{t_{5\ell}}, &(q_1,q_2,q_3)\ne (\mathbf{0}_d, 0, \mathbf{0}_K) 
    \\
    \sum_{i \in \gA_j}\exp(\gamma_i^n) - \exp(\gamma_j^*), &(q_1,q_2,q_3) = (\mathbf{0}_d, 0, \mathbf{0}_K)
\end{cases}
$$
and 
$$
W_t^n(j) = \begin{cases}
    \sum_{i \in \gA_j}\dfrac{\exp(\gamma_j^n)}{t!}(\Delta \alpha_{ij}^n)^{t_1}(\Delta\beta_{ij}^n)^{t_2}(\Delta\tau_{ij})^{t_3}, &(t_1,t_2,t_3) \ne (\mathbf{0}_d, 0, 0) \\ 
    \sum_{i \in \gA_j}\exp(\gamma_i^n) - \exp(\gamma_j^*)&(t_1,t_2,t_3) = (\mathbf{0}_d, 0, 0).
\end{cases}
$$
\textbf{Step 2.} In this step, we will prove that at least one among $U^n_{q_1,q_2,q_3}(j)/\gD_3(G_n, \overline{G}^*)$ does not vanish as $n \to \infty$. Assume, by contradiction, that all of them vanish as $n \to \infty$. By taking the summation of $|U^n_{q_1,q_2,q_3}(j)|/\gD_3(G_n, \overline{G}^*)$ for all $j \in \{[\ks]: |\gA_j|=1\}$, $q_2 = 0, q_3 = \mathbf{0}_K$, and  $q_1 \in \{e_1,e_2,\ldots,e_d\}$, where $e_i := (0,\ldots,0,\underbrace{1}_{i\text{-th}},0,\ldots,0) \in \mathbb{R}^d$, we achieve that 
$$
\dfrac{1}{\gD_3(\overline{G}_n, \overline{G}^*)}\sum_{j:|\gA_j| = 1}\sum_{i \in \gA_j}\exp(\gamma_i)\|\Delta a_{ij}^n\| \to 0.
$$
Similarity for $q_1 = \mathbf{0}_d$, $q_2 = 0$, and $q_3 \in \{e'_1, \ldots, e'_K\}$, where $e'_\ell := (0,\ldots,0,\underbrace{1}_{\ell\text{-th}},0,\ldots,0) \in \mathbb{R}^K$, we have 
$$
\dfrac{1}{\gD_3(\overline{G}_n, \overline{G}^*)}\sum_{j:|\gA_j| = 1}\sum_{i \in \gA_j}\exp(\gamma_i)|\Delta b_{ij}^n| \to 0.
$$
On the other hand, taking $(q_1,q_2,q_3) = (\mathbf{0}_d, 1, \mathbf{0}_K)$ and considering the three cases $|t_1|=1, t_2=1$, and $t_3 =1$ respectively, we obtain that 
$$
\dfrac{1}{\gD_3(\overline{G}_n, \overline{G}^*)}\sum_{j:|\gA_j| = 1}\sum_{i \in \gA_j}\exp(\gamma_i)\|\Delta \alpha_{ij}^n\| \to 0.
$$
$$
\dfrac{1}{\gD_3(\overline{G}_n, \overline{G}^*)}\sum_{j:|\gA_j| = 1}\sum_{i \in \gA_j}\exp(\gamma_i)|\Delta \beta_{ij}^n| \to 0.
$$
$$
\dfrac{1}{\gD_3(\overline{G}_n, \overline{G}^*)}\sum_{j:|\gA_j| = 1}\sum_{i \in \gA_j}\exp(\gamma_i)|\Delta \tau| \to 0.
$$
Combining the 5 results above, we deduce that 
\begin{equation}\label{eq: limit1_temperature}
    \sum_{j:|\gA_j| = 1}\sum_{i \in \gA_j}K_{ij}(1,1,1,1,1)/\gD_3(\overline{G}_n, \overline{G}^*) \to 0.
\end{equation}
By applying the same argument, we also deduce that 
\begin{equation}\label{eq: limit2_temperature}
    \sum_{j:|\gA_j| > 1}\sum_{i \in \gA_j}K_{ij}(2,2,2,2,2)/\gD_3(G_n, \overline{G}^*) \to 0
\end{equation} Additionally, for tuple $(q_1,q_2,q_3) = (\mathbf{0}_d, 0, \mathbf{0}_K)$, it follows that 
\begin{equation}\label{eq: limit3_temperature}
    \dfrac{1}{\gD_3(\overline{G}_n, \overline{G}^*)}\cdot \left(\sum_{j=1}^{\ks}
    \left|\sum_{i\in\gA_j} \exp(\gamma_i) - \exp(\gasi)\right|\right) \to 0
\end{equation}
It is induced from the limits in equations \eqref{eq: limit1_temperature}, \eqref{eq: limit2_temperature}, \eqref{eq: limit3_temperature} that $1 = \gD_3(\overline{G}_n, \overline{G}^*)/\gD_3(\overline{G}_n, \overline{G}^*) \to 0$ when $n \to \infty$, which is a contradiction. Thus, at least one among $U^n_{q_1,q_2,q_3}(j)/\gD_3(\overline{G}_n, \overline{G}^*)$ does not go to zero when $n \to \infty$. 

\textbf{Step 3.} Now, we denote $\overline{m}_n$ as the maximum of the absolute values of $U^n_{q_1,q_2,q_3}(j)/\gD_3(\overline{G}_n, \overline{G}^*)$ and $W^n_t(j)/\gD_3(\overline{G}_n, \overline{G}^*)$. Since at least one among $U^n_{q_1,q_2,q_3}(j)/\gD_3(\overline{G}_n, \overline{G}^*)$ does not go to zero when $n \to \infty$, we deduce that $\overline{m}_n \not\to 0$, and equivalently $1/\overline{m}_n \not\to \infty$. Then we denote 
\begin{align*}
    U^n_{q_1,q_2,q_3}(j)/\gD_3(\overline{G}_n, \overline{G}^*) &\to \zeta_{q_1,q_2,q_3}(j) \\
    W^n_t(j)/\gD_3(\overline{G}_n, \overline{G}^*) &\to \eta_t(j),
\end{align*}
as $n \to \infty$. By invoking the Fatou's lemma, we have that 
$$
0 = \lim_{n \to \infty} \frac{\mathbb{E}_X[2d_V(\overline{p}_{\overline{G}_n}(\cdot|X), \overline{p}_{\overline{G}^*}(\cdot|X))]}{\overline{m}_n \gD_3(\overline{G}_n, \overline{G}^*)} \ge \int \sum_{s=1}^K \liminf_{n \to \infty} \left| \frac{\overline{p}_{\overline{G}_n}(Y=s|X) - \overline{p}_{\overline{G}^*}(Y=s|X)}{m_n \gD_3(\overline{G}_n, \overline{G}^*)} \right| \mathrm{d} X \ge 0.
$$
Consequently, it indicates that $|\overline{p}_{\overline{G}_n}(Y=s|X) - \overline{p}_{\overline{G}^*}(Y=s|X)|/[\overline{m}_n \gD_3(\overline{G}_n, \overline{G}^*)]\to 0$ for any $s \in [K]$ and almost surely $X$. This result is equivalent to 
\begin{equation}\label{eq: step 3}
    T_n(s)/[m_n\gD_3(\overline{G}_n, \overline{G}^*)] \to 0,
\end{equation}
as $n \to \infty$, for any $s \in [K]$. Putting the results in equations \eqref{eq: a_n + e_n,1}, \eqref{eq: b_n + e_n,2}, and \eqref{eq: step 3} together, we have 
\begin{align*}
\sum_{j=1}^{k^*}\sum_{t=0}^{1+\mathbf{1}_{\{|\gA_j|>1\}}}\Bigg[\sum_{(q_1,q_2,q_3)\in \gI^{t}}\zeta_{q_1,q_2,q_3}(j)X^{q_1}\sigma^{(q_2)}(g_j (X,\alpha_j^*,\beta_j^*, \tau_j^*)) \dfrac{\partial^{|q_3|}f}{\partial h^{q_3}}(Y=s|X;a_j^*, b_j^*)\\
+\eta_t(j)\sigma^{(|t|)}(g_j(X;\alpha_j^*, \beta_j^*, \tau^*)) \overline{p}_{\overline{G}^*}(Y=s|X)\Bigg]\times \dfrac{\partial^{|t_1|+t_2+t_3}g_j}{\partial\alpha^{t_1}\partial\beta^{t_2}\partial\tau^{t_3}}(X;\alpha_j^*, \beta_j^*, \tau^*) = 0,
\end{align*}
for almost surely $(X,Y)$. It is worth noting that parameters $\alpha_j^*, \beta_j^*$ are pair-wise distinct, then, the set
\begin{align*}
    \Bigg\{ &X^{q_1}\sigma^{(q_2)}(g_j (X,\alpha_j^*,\beta_j^*, \tau_j^*)) \dfrac{\partial^{|q_3|}f}{\partial h^{q_3}}(Y=s|X;a_j^*, b_j^*)\dfrac{\partial^{|t_1|+t_2+t_3}g_j}{\partial\alpha^{t_1}\partial\beta^{t_2}\partial\tau^{t_3}}(X;\alpha_j^*, \beta_j^*, \tau^*) ; \\
    & \sigma^{(|t|)}(g_j(X;\alpha_j^*, \beta_j^*, \tau^*)) \overline{p}_{\overline{G}^*}(Y=s|X) \dfrac{\partial^{|t_1|+t_2+t_3}g_j}{\partial\alpha^{t_1}\partial\beta^{t_2}\partial\tau^{t_3}}(X;\alpha_j^*, \beta_j^*, \tau^*) \Bigg\}
\end{align*}
is linearly independent for any $j \in [\ks]$, $0 \le q_2 + |q_3| \le 2, 0 \le |t| \le 2$ and $j \in [\ks]$. Therefore, we obtain that all the coefficients $\zeta_{q_1,q_2,q_3} = \eta_{t} = 0$, which contradicts the fact that at least one among $\zeta_{q_1,q_2,q_3}$ is different from zero. 

\section{Proofs of Auxiliary Results}
\label{appendix:auxiliary_results}

\subsection{Proof of Proposition~\ref{prop:density_rate}}
\label{appendix:density_rate}

\subsubsection{Key Results}

Initially, we introduce some notations that will be used in our proof.
Let $P_k(\Theta) := \{ p_G(Y\mid X) : G \in \gG_k(\Theta)\}$ be the set of conditional
density functions of all mixing measures in $\gG_k(\Theta)$.
In addition, we define $N(\varepsilon, P_k(\Theta), \|\cdot\|_\infty)$ as the covering number \cite{Vandegeer-2000} of metric space $(P_k(\Theta), \|\cdot\|_\infty)$, while
$H_B(\varepsilon, P_k(\Theta), h)$ is the bracketing entropy \cite{Vandegeer-2000} of $P_k(\Theta)$
under the Hellinger distance. Then, the following result provides the upper bound of these quantities:

\begin{lemma}\label{lemma:n-covering}
For any bounded set $\Theta$ and $\varepsilon \in (0,1/2)$, we have
\begin{itemize}
    \item[(i)] $\log N(\varepsilon, P_k(\Theta),\|\cdot\|_\infty) \lesssim Kdk \log(1/\varepsilon)$;
    \item[(ii)] $H_B(\varepsilon, P_k(\Theta), d_{H}) \lesssim Kdk \log(1/\varepsilon)$.
\end{itemize}
\end{lemma}
\begin{proof}[Proof of Lemma~\ref{lemma:n-covering}]

\textbf{Part (i)}  
Firstly, we introduce the parameter sets
\[
\Omega := \{ (a,b) \in \mathbb{R}^{d\times K} \times \mathbb{R}^{K} : (\gamma, \alpha, \beta,a,b) \in \Theta\}.
\]
\[
\Delta := \{ (\gamma, \alpha, \beta) \in \mathbb{R}\times\mathbb{R}^d\times\mathbb{R} : (\gamma, \alpha, \beta,a,b)\in\Theta\},
\]
Because $\Theta$ is a compact set, both $\Omega$ and $\Theta$ are compact. Consequently, each admits an $\varepsilon$-cover; denote them by $\Omega_\varepsilon$ and $\Delta_\varepsilon$ respectively.
It can be validated that
\[
|\Omega_\varepsilon| \lesssim \mathcal{O}(\varepsilon^{-K(d+1)k}),
\qquad
|\Delta_\varepsilon| \lesssim \mathcal{O}(\varepsilon^{-(d+2)k}).
\]

Next, consider a mixing measure $
G = \sum_{i=1}^{k'} \exp(\gamma_{i})\delta_{(\alpha_{i}, \beta_{i},a_i,b_i)} \in \gG_k(\Theta),
 \forall \, k' \in [k].
$ Using the $\varepsilon$-covers introduced above, we associate to $G$ two approximating measures.

First, let $(\bar{a}_i, \bar{b}_i) \in \Omega_\varepsilon$ be the point in the cover that is closest to $(a_i, b_i)$ for each $i \in [k']$, and define
\[
\widetilde{G} = \sum_{i=1}^{k'} \exp(\gamma_{i}) \delta_{(\alpha_{i},\beta_{i},\bar{a}_i,\bar{b}_i)}.
\]

Next, take $(\bar{\gamma}_i, \bar{\alpha}_i, \bar{\beta}_i) \in \Delta_\varepsilon$ as the closest cover point to $(\gamma_i, \alpha_i, \beta_i)$, and set
\[
\bar{G} = \sum_{i=1}^{k'} \exp(\bar{\gamma}_{i}) \delta_{(\bar{\alpha}_{i},\bar{\beta}_{i},\bar{a}_i,\bar{b}_i)}.
\]

It follows directly from the construction that $p_{\bar{G}}$ belongs to the family
\[
\mathcal{H} := \left\{
    p_G \in P_k(\Theta)
    :
    (\bar{\gamma}_{i},\bar{\alpha}_{i}, \bar{\beta}_{i}) \in \Delta_\varepsilon,\;
    (\bar{a}_i,\bar{b}_i)\in\Omega_\varepsilon,\;
    \forall i \in [k']
\right\}.
\]

We now verify that $\mathcal{H}$ is an $\varepsilon$-cover of the metric space
$(P_k(\Theta),\|\cdot\|_\infty)$ but not necessarily the minimal one.
For that purpose, we bound the uniform distance
$\| p_G - p_{\overline{G}} \|_\infty$.
Applying the triangle inequality gives
\[
    \| p_G - p_{\overline{G}}\|_\infty
    \le
    \| p_G - p_{\widetilde{G}}\|_\infty
    +
    \| p_{\widetilde{G}} - p_{\overline{G}}\|_\infty.
\]

For any mixing measure 
\(G = \sum_{i=1}^{k'} \exp(\gamma_i)\delta_{(\alpha_i,\beta_i,a_i,b_i)}\),
we recall that the corresponding conditional density is
\[
p_G(Y = s\mid X)
=
\sum_{i=1}^{k'}
\pi(Y\mid X;\gamma_i,\alpha_i,\beta_i)\,
f(Y = s\mid X; a_i,b_i),
\]
where \[\pi(Y\mid X;\gamma_i,\alpha_i,\beta_i) = \frac{\exp(\gamma_i)\sigma((\alpha_i)^{\top}X+\beta_i)}{\sum_{j=1}^{k'}\exp(\gamma_j)\sigma((\alpha_j)^{\top}X+\beta_j)}\nonumber\text{, and } f(Y = s\mid X; a_i,b_i) = \frac{\exp((a_{is})^{\top}X+b_{is})}{\sum_{\ell=1}^{K}\exp((a_{i\ell})^{\top}X+b_{i\ell})}\]
Turning to the first term on the right-hand side,
\begin{align*}
    &\|p_{\widetilde{G}} - p_G \|_\infty\\
    & = \sum_{s=1}^K \sup_{X}
        \left|
            p_G(Y=s\mid X) - p_{\widetilde{G}}(Y=s\mid X)
        \right|. \\
    &= \sum_{s=1}^K \sup_{X} \left|
    \sum_{i=1}^{k'}
        \pi(Y\mid X;\gamma_i,\alpha_i,\beta_i)\;
        f(Y=s\mid X; a_i,b_i)
        -
    \sum_{i=1}^{k'}
        \pi(Y\mid X;{\gamma}_i,
                        {\alpha}_i,
                        {\beta}_i)\;
        f(Y=s\mid X; \bar{a}_i,\bar{b}_i)
\right| \\
& \le \sum_{s=1}^K
    \sup_X
    \sum_{i=1}^{k'}
        \left|
            \pi(Y\mid X;\gamma_i,\alpha_i,\beta_i)
        \right|
        \cdot
        \left|
            f(Y=s\mid X; a_i,b_i)
            -
            f(Y=s\mid X; \bar{a}_i,\bar{b}_i)
        \right| \\
    & \le \sum_{s=1}^K
    \sup_X
    \sum_{i=1}^{k'}
        \left|
            f(Y=s\mid X; a_i,b_i)
            -
            f(Y=s\mid X; \bar{a}_i,\bar{b}_i)
        \right|. 
\end{align*}
The differentiability of $f$ implies it is Lipschitz continuous with some constant $L>0$.
In addition, since $\mathcal{X}$ is a bounded set, we may assume that
$\|X\| \le B$ for some constant $B>0$.
This yields, we have
\[
\| p_{\widetilde{G}} - p_G \|_\infty
    \le \sum_{s=1}^K \sum_{i=1}^{k'}
        L \cdot
        \left(
            \| a_{is} - \bar{a}_{is} \|\cdot\|X\|
            + \| b_{is} - \bar{b}_{is} \|
        \right)
\]
\[
\le K k' L (B\varepsilon + \varepsilon)
    \lesssim \varepsilon.
\]

Similarly, we bound the second term as
\[
\| p_{\overline{G}} - p_{\widetilde{G}} \|_\infty
    = \sum_{s=1}^K
        \sup_X
        \sum_{i=1}^{k'}
        \left|
                \pi(Y\mid X;\overline{\gamma}_i,
                                \overline{\alpha}_i,
                                \overline{\beta}_i)
                -
                \pi(Y\mid X;\gamma_i,\alpha_i,\beta_i)
        \right| f(Y=s|X;\overline{a}_i; \overline{b}_i).
\]

\[
\le \sum_{s=1}^K \sum_{i=1}^{k'}
    L \cdot
    \left(
        \|\alpha_{i} - \bar{\alpha_{i}}\|\cdot\|X\|
        +
        |\beta_{i} - \bar{\beta}_{i}|
        + |\gamma_{i} - \bar{\gamma}_{i}|
    \right)
\]
\[
\le K k' L (B\varepsilon + \varepsilon + \varepsilon)
    \lesssim \varepsilon.
\]
Putting these bounds together gives
\(
\| p_G - p_{\overline{G}} \|_\infty
    \lesssim \varepsilon.
\)
Hence $\mathcal{H}$ is an $\varepsilon$–cover of the metric space $(P_k(\Theta),\|\cdot\|_\infty)$.
From the definition of covering numbers, we therefore obtain
\[
N(\varepsilon, P_k(\Theta),\|\cdot\|_\infty)
    \le |\mathcal{H}|
    = |\Omega_\varepsilon| \times |\Delta_\varepsilon|
    = \mathcal{O}(\varepsilon^{-K(d+1)k})
      \times
      \mathcal{O}(\varepsilon^{-(d+2)k})
\]
\[
= \mathcal{O}(\varepsilon^{-(Kd + K + d + 2)k}),
\]
which is equivalent to
\[
\log N(\varepsilon, P_k(\Theta), \|\cdot\|_\infty)
    \lesssim Kdk\log(1/\varepsilon)\lesssim Kdk\log(1/\varepsilon).
\]

\noindent\textbf{Part (ii)} Fix $\varepsilon>0$ and choose $\eta \le \varepsilon$ (to be specified later)..
Let ${p_1,p_2,\ldots,p_N}$ be an $\eta$-cover of $P_k(\Theta)$ under the $|\cdot|\infty$ norm, with $N := N(\eta, P_k(\Theta), |\cdot|_\infty)$.
We now construct brackets of the form $[L_i(Y\mid X), U_i(Y\mid X)]$ for all $i\in[N]$ as below:
\[
L_i(Y\mid X) := \max\{p_i(Y\mid X)-\eta,\,0\},\qquad
U_i(Y\mid X) := \min\{p_i(Y\mid X)+\eta,\,1\}.
\]
From this construction it follows that
$P_k(\Theta) \subset \bigcup_{i=1}^N [L_i(Y\mid X),U_i(Y\mid X)]$
and $U_i(Y\mid X)-L_i(Y\mid X)\le \min\{2\eta,1\}$.
Moreover
\[
\|U_i(\cdot\mid X)-L_i(\cdot\mid X)\|_1
= \sum_{\ell=1}^K \bigl[U_i(Y=\ell\mid X)-L_i(Y=\ell\mid X)\bigr]
\le 2K\eta.
\]
Recall that $H_B(2K\eta, P_k(\Theta),|\cdot|_1)$ is defined as the logarithm of the minimal number of brackets with $L_1$-width no larger than $2K\eta$ needed to cover $P_k(\Theta)$. For any $p\in P_k(\Theta)$, there exists some $p_i$ in the $\eta$-cover such that $|p-p_i|\infty \le \eta$. The corresponding bracket $[L_i,U_i]$ then satisfies
\[
p \in [L_i,U_i]
\qquad\text{and}\qquad
\|U_i - L_i\|_1 \le 2K\eta.
\]
Hence each $p_i$ in the $\eta$-cover yields a valid $2K\eta$-bracket.
Consequently, the number of required brackets does not exceed $N(\eta, P_k(\Theta),|\cdot|\infty)$, giving
\[
H_B(2K\eta, P_k(\Theta),\|\cdot\|_1)
\;\le\;
\log N(\eta, P_k(\Theta),\|\cdot\|_\infty)
\;\lesssim\;
 Kdk\log(1/\eta),
\]
where the final inequality uses the bound from Part (i).
Taking $\eta=\varepsilon/(2K)$ therefore yields $H_B(\varepsilon, P_k(\Theta), |\cdot|_1)\lesssim Kdk \log(1/\varepsilon)$.
Since the Hellinger distance satisfies $h \le |\cdot|_1$, we obtain
\[
H_B(\varepsilon, P_k(\Theta), d_{H}) \;\lesssim Kdk\; \log(1/\varepsilon).
\]
\end{proof}
Given a mixing set $\gG_k(\Theta)$, we introduce the associated conditional density families
\[
\overline{P}_k(\Theta)\;:=\;\{\,p_{(G+G^*)/2}(\,\cdot\mid X):\,G\in \gG_k(\Theta)\,\},
\qquad
\overline{P}_k^{1/2}(\Theta)\;:=\;\{\,p_{(G+G^*)/2}^{1/2}(\,\cdot\mid X):\,G\in \gG_k(\Theta)\,\}.
\]
For any $\xi>0$, we define the Hellinger ball centered around the true conditional density
$p_{\Gs}(\,\cdot\mid X)$ and intersected with the set $\overline{P}_k^{1/2}(\Theta)$ as
\[
\overline{P}_k^{1/2}(\Theta,\xi)
:=\Bigl\{
p^{1/2}\in \overline{P}_k^{1/2}(\Theta):\;
\bbE_X\!\big[d_{H}\!\big(p(\,\cdot\mid X),\,p_{\Gs}(\,\cdot\mid X)\big)\big]\le \xi
\Bigr\}.
\]

As before, $H_B(\varepsilon,\mathcal{F},|\cdot|)$ denotes the bracketing entropy of a function class $\mathcal{F}$ with respect to the norm $|\cdot|$. Following \cite{Vandegeer-2000},
we define the localized bracketing integral to capture for $\xi>0$ the size of the Hellinger ball as
\[
J_B\!\big(\xi,\overline{P}_k^{1/2}(\Theta,\xi)\big)
:=\int_{\xi^2/2^{13}}^{\xi} \sqrt{\,H_B\!\big(t,\overline{P}_k^{1/2}(\Theta,\xi),\|\cdot\|\big)\,}\;dt \lor \xi.
\]
where $t \lor \xi := \max\{t, \xi\}$. We now state a key result from \cite{Vandegeer-2000} on density estimation rates, adapted to our notation.


\begin{lemma}\label{lemma:geer-2000}
(Theorem 7.4, \cite{Vandegeer-2000}) \\
Let $\Psi(\xi)$ satisfy $\Psi(\xi) \ge J_B(\xi,\overline{P}_k^{1/2}(\Theta,\xi))$ with $\Psi(\xi)/\xi^2$ non-increasing in $\xi$. Then for a universal constant $c$ and any sequence $(\xi_n)$ fulfilling $\sqrt{n}\xi_n^2 \ge c\Psi(\xi_n)$, we have
\[
\mathbb{P}\Big( \mathbb{E}_X[d_{H}(p_{\widehat{G}_n}(\cdot|X), p_{G^*}(\cdot|X))] > \xi \Big)
\le c\exp\Big( -\frac{n\xi^2}{c^2} \Big),
\]
for any $\xi \ge \xi_n$.
\end{lemma}

The proof of this lemma appears in \cite{Vandegeer-2000}.

\subsubsection{Main Proof}

Observe that
\(
H_B\!\left(t,\,\overline{P}_{k}^{1/2}(\Theta,t),\,\|\cdot\|\right)
\;\le\;
H_B\!\left(t,\,P_k(\Theta),\,d_{H}\right),\forall \, t>0
\). 
Consequently,
\begin{align*}
J_B\!\left(\xi,\,\overline{P}_{k}^{1/2}(\Theta,\xi)\right)
&=
\int_{\xi^2/2^{13}}^{\xi}
H_B^{1/2}\!\left(t,\,\overline{P}_{k}^{1/2}(\Theta,t),\,\|\cdot\|\right)\,dt  \lor \xi
\\[4pt]
&\le
\int_{\xi^2/2^{13}}^{\xi}
H_B^{1/2}\!\left(t,\,P_k(\Theta),\,d_H\right)\,dt  \lor \xi
\qquad\text{(monotonicity)}
\\[4pt]
&\lesssim
\int_{\xi^2/2^{13}}^{\xi}
\sqrt{Kdk}\sqrt{\log(1/t)}\,dt  \lor \xi
\qquad\text{(using Lemma~\ref{lemma:n-covering} )}
\\[4pt]
&\lesssim\xi\,\sqrt{Kdk\log(1/\xi)}\\
\end{align*}
Define
\[
\Psi(\xi) \;=\; \sqrt{Kdk}\cdot\xi\,\sqrt{\log(1/\xi)}
\quad\text{and}\quad
\xi_n=\sqrt{Kdk\log(n)/n}.
\]
As $\Psi(\xi)\ge J_B(\xi,\overline{P}_k^{1/2}(\Theta,\xi))$ and $\Psi(\xi)/\xi^2$ is non-increasing, Lemma~\ref{lemma:geer-2000} gives
\[
\mathbb{P}\!\left(
\mathbb{E}_X\!\left[d_{H}\!\left(p_{\widehat{G}_n}(\cdot\mid X),\,p_{G^*}(\cdot\mid X)\right)\right]
>C\sqrt{Kdk\log(n)/n}
\right)
\;\le\;
\,\exp\!\left(
-c\log(n)
\right),
\]
where $C$ and $c$ are universal positive constants depending only on $\Theta$.
This completes the proof.

\subsection{Proof of Proposition~\ref{prop:Identifiability}}
\label{appendix:Identifiability}

Under the assumption, for every $s \in [K]$ and $\mu$-almost surely $\mathbf{X} \in \mathcal{X}$, we have

\begin{align}
&\sum_{i=1}^{k}\frac{\exp(\gamma_i)\sigma((\alpha_i)^{\top}X+\beta_i)}{\sum_{j=1}^{k}\exp(\gamma_j)\sigma((\alpha_j)^{\top}X+\beta_j)
}
     \cdot \frac{\exp((a_{is})^{\top}X+b_{is})}{\sum_{\ell=1}^{K}\exp((a_{i\ell})^{\top}X+b_{i\ell})}\nonumber\\
&=\quad \sum_{i=1}^{k'}\frac{\exp(\gamma'_i)\sigma((\alpha'_i)^{\top}X+\beta'_i)}{\sum_{j=1}^{k'}\exp(\gamma'_j)\sigma((\alpha'_j)^{\top}X+\beta'_j)
}
     \cdot \frac{\exp((a'_{is})^{\top}X+b'_{is})}{\sum_{\ell=1}^{K'}\exp((a'_{i\ell})^{\top}X+b'_{i\ell})}
\end{align} 
\cite{grun2008iden}
show that multinomial logistic mixtures are identifiable, implying that any two mixing measures \(G\) and \(G'\) must have the same number of experts and identical gating sets, i.e., \(k = k'\) and

$$
\left\{\frac{\exp(\gamma_i)\sigma((\alpha_i)^{\top}X+\beta_i)}{\sum_{j=1}^{k}\exp(\gamma_j)\sigma((\alpha_j)^{\top}X+\beta_j)
} : i \in [k] \right\} = \left\{ \frac{\exp(\gamma'_i)\sigma((\alpha'_i)^{\top}X+\beta'_i)}{\sum_{j=1}^{k'}\exp(\gamma'_j)\sigma((\alpha'_j)^{\top}X+\beta'_j)
} : i \in [k'] \right\},
$$

for almost surely $\mathbf{X} \in \mathcal{X}$. Without loss of generality, we assume that

\begin{equation}
    \label{20}
    \frac{\exp(\gamma_i)\sigma((\alpha_i)^{\top}X+\beta_i)}{\sum_{j=1}^{k}\exp(\gamma_j)\sigma((\alpha_j)^{\top}X+\beta_j)
} = \frac{\exp(\gamma'_i)\sigma((\alpha'_i)^{\top}X+\beta'_i)}{\sum_{j=1}^{k'}\exp(\gamma'_j)\sigma((\alpha'_j)^{\top}X+\beta'_j)
} 
\end{equation}

for any $i \in [k]$.
The above result leads to $\gamma_i = \gamma'_i$, $\alpha_i=\alpha_i'$ and $\beta_i=\beta_i'$ for all $i \in [k] $. \\
Thus, the equation~\ref{20} becomes
\begin{equation} \label{eq:41}
\sum_{i=1}^k \exp(\gamma_{i}) u(Y=s|X; \alpha_i, \beta_i, a_i, b_i) = \sum_{i=1}^k  \exp(\gamma_{i}) u(Y=s|X; \alpha'_i, \beta'_i, a'_i, b'_i),
\end{equation}
for any $s \in [K]$ and almost surely $\mathbf{X} \in \mathcal{X}$, where $u(Y=s|X; \alpha_, \beta, a, b) := \sigma(\alpha^\top X + \beta) f(Y=s|X; a, b)$

Next, we will consider $r$ subsets of the set $[k]$, denoted by $\mathcal{S}_1, \mathcal{S}_2, \dots, \mathcal{S}_r$ satisfying the following property: $\gamma_{i} = \gamma_{i'}$ for any $i, i' \in \mathcal{S}_t$ for $t \in [r]$. Therefore, equation~\ref{eq:41} can be rewritten as
$$
\sum_{t=1}^r \sum_{i \in \mathcal{S}_t} \exp(\gamma_{i}) u(Y=s|X; \alpha_i, \beta_i, a_i, b_i) = \sum_{t=1}^r \sum_{i \in \mathcal{S}_t} \exp(\gamma_{i}) u(Y=s|X; \alpha'_i, \beta'_i, a'_i, b'_i).
$$

for any $s \in [K]$ and almost surely $\mathbf{X} \in \mathcal{X}$. \\
It follows from the above equation that for each $t \in [r]$, we get $\{ (a_{il})^{\top}X+b_{il})_{l=1}^K : i \in \mathcal{S}_t \} = \{ (a'_{il})^{\top}X+b'_{il})_{l=1}^K : i \in \mathcal{S}_t \}$ for almost surely $\mathbf{X} \in \mathcal{X}$. This leads to
$$
\{ (a_{i1}, \dots, a_{iK}, b_{i1}, \dots, b_{iK}) : i \in \mathcal{S}_t \} = \{ (a'_{i1}, \dots, a'_{iK}, b'_{i1}, \dots, b'_{iK}) : i \in \mathcal{S}_t \}.
$$
Again, we may assume that $(a_{i1}, \dots, a_{iK}, b_{i1}, \dots, b_{iK}) = (a'_{i1}, \dots, a'_{iK}, b'_{i1}, \dots, b'_{iK})$ for any $i \in \mathcal{S}_t$.\\
As a result, we achieve that $G \equiv G'$, which completes the proof.
\section{Experimental Details}\label{appendix:numerical}
In this appendix, we present numerical experiments that empirically validate our theoretical results on the convergence rates of maximum likelihood estimation for the standard sigmoid gating multinomial logistic MoE model without temperature parameter in the Appendix~\ref{experiment:w/o temp}, and we compare it with that for the softmax gating multinomial logistic MoE model in the Appendix~\ref{experiment:compare_gating}. Meanwhile, in the Appendix~\ref{experiment:w/ temp}, for the temperature enhanced gating variant, we empirically demonstrate the advantage of using a Euclidean affinity score instead of the standard inner product affinity score for parameter estimation in the sigmoid gating multinomial logistic MoE model.

\textbf{Synthetic data.} For each sample size $n$, we generate independent and identically distributed samples ${(X_i, Y_i)}_{i=1}^n$ by first drawing $X_i$'s from the uniform distribution over $[0,1]$ and then sampling $Y_i$'s from the true conditional density $p_{\Gs}(Y=s|X)$, $\widetilde{p}_{\widetilde{G}^*}(Y=s|X)$, and $\overline{p}_{\overline{G}^*}(Y=s|X)$ corresponding to the MoE model specified in each theorem configuration. 

\textbf{Maximum likelihood estimation.} A widely used approach for computing the maximum likelihood estimator $\widehat{G}$ (or $\widetilde{G}, \overline{G}$)
for each set of samples is to use the Expectation-Maximization (EM) algorithm \cite{dempster1977maximum}. Since closed form updates for the gating parameters are not available in the maximization step, we instead employ an EM based numerical procedure previously adopted in \cite{chamroukhi2009time}. We set the convergence tolerance to $\varepsilon = 10^{-6}$ and run the algorithm for a maximum of 1000 iterations.  

\textbf{Experimental design.} Our empirical study considers three experimental configurations, each directly aligned with the theoretical settings developed in the main paper. For each configuration, we generate 20 independent datasets across 20 distinct sample sizes ranging from $10^4$ to $10^5$. To ensure consistency and facilitate comparison across experiments, we adopt a model architecture with two true experts $(\ks=2)$ and two classes ($K=2$), while fitting models with three and four experts in the experimental settings.

\textbf{Initialization.} For each $k \in \{\ks+1, \ks+2\}$, the indices in the set ${1,2,...,k}$ are randomly assigned to $\ks$ Voronoi cells denoted by $\gA_1, \ldots,\gA_{\ks}$,  ensuring that each cell contains at least one element. This assignment is repeated independently for each replication. Subsequently, the model parameters are initialized by sampling from Gaussian distributions centered at their true values with a small variance to ensure stable convergence. 

\subsection{Sigmoid Gating without Temperature}\label{experiment:w/o temp}
\begin{table}[!ht]
\centering
\caption{True parameters for the comparative analysis of sigmoid and softmax gating functions.}
\begin{tabular}{cccc}
\hline
 & \textbf{Gating Parameters} & \multicolumn{2}{c}{\textbf{Expert Parameters}} \\
 & & Class 1 $(a^*, b^*)$ & Class 2 $(a^*, b^*)$ \\ \hline
\multicolumn{4}{l}{\textbf{Sigmoid Gating ($p_{G^*}$)}} \\
Expert 1 & $(\gamma_0^*, \alpha_0^*,\beta_0^*) = (0.2, -1.0, -0.5)$ & $(0.0, -1.0)$ & $(-1.0, 1.0)$ \\
Expert 2 & $(\gamma_0^*, \alpha_0^*,\beta_0^*) = (-0.3, 1.0, 0.5)$ & $(0.0, 0.0)$ & $(0.0, 2.0)$ \\
\hline
\multicolumn{4}{l}{\textbf{Softmax Gating ($g_{G^*}$)}} \\
Expert 1 & $(\beta_{01}^*, \beta_{11}^*) = (0.2, -1.0)$ & $(-1.0, 0.0)$ & $(1.0, -1.0)$ \\
Expert 2 & $(\beta_{02}^*, \beta_{12}^*) = (-0.3, 1.0)$ & $(0.0, 0.0)$ & $(2.0, 0.0)$ \\ \hline
\end{tabular}

\label{table:true params compare}
\end{table}
The problem setting is defined as in Theorem \eqref{theorem:without_temperature}. The values corresponding to the true parameters are detailed in Table \ref{table:true params compare}.    
As illustrated in Figure \eqref{fig:convergence_rate}, the maximum likelihood estimator MLE $\widehat{G}_n$ exhibits empirical convergence to the true mixing measure $G^*$ under the Voronoi metric $\gD_1$ at the rate of orders $\gO(n^{-0.49})$ and $\gO(n^{-0.54})$ when $k=3$ and $k=4$, respectively. This result strongly supports our theoretical result in Theorem~\ref{theorem:without_temperature}. 

\subsection{Convergence Analysis between Sigmoid Gating and Softmax Gating}\label{experiment:compare_gating}

\textbf{Model Definitions.} Next, we empirically compare the convergence behavior of the standard sigmoid-gating MLMoE model in Theorem~\eqref{theorem:without_temperature} with the softmax-gating MLMoE model, whose conditional density $g_{G^*}(Y=s|X)$ is defined in Equation (1) of \cite{nguyen2024general}. We configure the true parameters of the model to satisfy the conditions of {Regime 2} from \cite{nguyen2024general}, where at least one expert has vanishing coefficients for the covariate $X$ ($b_{i\ell}^* = 0$ for all $\ell$), a regime proven to induce slow convergence rates for parameter estimation under standard softmax gating.

\textbf{True Parameter Configuration.} The true parameter values for both models are detailed in Table~\ref{table:true params compare}.

\textbf{Empirical Convergence Results.} 
Figure~\ref{fig:convergence_rate} presents the empirical convergence of the maximum likelihood estimator for both gating functions when fitting $k = 3$ and $k = 4$ experts to the true $k^* = 2$ expert model. The sigmoid gating model (measured under $\mathcal{D}_1$) demonstrates faster convergence rates: $\mathcal{O}(n^{-0.49})$ for $k=3$ and $\mathcal{O}(n^{-0.54})$ for $k=4$, while the softmax gating model (measured under $\mathcal{D}_2$ from \cite{nguyen2024general}) converges at slower rates $\mathcal{O}(n^{-0.30})$ ($k=3$) and $\mathcal{O}(n^{-0.28})$ ($k=4$).

\subsection{Sigmoid Gating with Temperature}\label{experiment:w/ temp}
\textbf{Sigmoid gating with inner product affinity score.}


\begin{table}[!ht]
\centering   
\caption{True parameters under the MLMoE model with temperature and inner product affinity score}
\begin{tabular}{cccc}
\hline
    & Gating Parameters                    & \multicolumn{2}{c}{Expert Parameters}            \\
    & $(\gamma^*, \alpha^*, \beta^*, \tau^*)$ & Class 1 $(a^*, b^*)$                & Class 2 $(a^*, b^*)$               \\ \hline
Expert $1$ & $(0, 1.0, 0, 0.1)$  
           & $(0, 0)$ & $(1.86, 0.963)$ \\
Expert $2$ & $(0, 1.001, 0, 0.1)$ 
           & $(0, 0)$ & $(1.87, 0.964)$ \\ \hline
\end{tabular}

\label{table:true params w temp inner}
\end{table}
We consider the same experimental setup as in Theorem~\eqref{theorem:with_temperature_inner}. 
The true values of the model parameters are summarized in Table~\ref{table:true params w temp inner}. 
Two subfigures~\eqref{fig:conv_d2_k3}, \eqref{fig:conv_d2_k4} show the estimation error of the maximum likelihood estimator $\widetilde{G}_n$ under the Voronoi metric $\gD_{2,2}$. 
We observe empirical convergence toward the true mixing measure $G^*$ with rates of order $\gO(n^{-0.11})$ and $\gO(n^{-0.15})$ for $k=3$ and $k=4$, respectively. 
These observations are consistent with the theoretical guarantee in Theorem~\ref{theorem:with_temperature_inner}.

\textbf{Sigmoid gating with Euclidean affinity score.}
\begin{table}[!ht]
\centering    
\caption{True parameters under the MLMoE model with temperature and Euclidean affinity score}
\begin{tabular}{cccc}
\hline
    & Gating Parameters                    & \multicolumn{2}{c}{Expert Parameters}            \\
    &                                      & Class 1                 & Class 2                \\ \hline
Expert $1$ & $(\gamma_0^*, \alpha_0^*,\beta_0^*,\tau^*) = (1,-5,-0.5,2.0)$  
           & $(a_{11}^*, b_{11}^*) = (-1,2)$ & $(a_{12}^*, b_{12}^*) = (0,0)$ \\
Expert $2$ & $(\gamma_0^*, \alpha_0^*,\beta_0^*,\tau^*) = (-1,5,0.5,2.0)$ 
           & $(a_{21}^*, b_{21}^*) = (1,-1)$ & $(a_{22}^*, b_{22}^*) = (0,0)$ \\ \hline
\end{tabular}

\label{table:true params w temp}
\end{table}

The experimental configuration follows Theorem~\eqref{theorem:with_temperature_euclidean}. 
Table~\ref{table:true params w temp} reports the corresponding true parameter values used in this setting.    
As shown in two subfigures~\eqref{fig:conv_d2_k3},\eqref{fig:conv_d2_k4}, the estimator $\overline{G}_n$ converges empirically to $G^*$ under the Voronoi metric $\gD_3$. 
The observed rates are of order $\gO(n^{-0.52})$ and $\gO(n^{-0.50})$ for $k=3$ and $k=4$, respectively, indicating stable convergence behavior. 
This empirical evidence aligns well with the theoretical prediction in Theorem~\ref{theorem:with_temperature_euclidean}.
\bibliography{references}

@inproceedings{chi_representation_2022,
	title = {On the Representation Collapse of Sparse Mixture of Experts},
	booktitle = {Advances in {Neural} {Information} {Processing} {Systems}},
	author = {Chi, Zewen and Dong, Li and Huang, Shaohan and Dai, Damai and Ma, Shuming and Patra, Barun and Singhal, Saksham and Bajaj, Payal and Song, Xia and Mao, Xian-Ling and Huang, Heyan and Wei, Furu},
	editor = {Oh, Alice H. and Agarwal, Alekh and Belgrave, Danielle and Cho, Kyunghyun},
	year = {2022},
}

@inproceedings{nguyen2024sigmoid,
  title={Sigmoid Gating is More Sample Efficient than Softmax Gating in Mixture of Experts},
  author={Nguyen, Huy and Ho, Nhat and Rinaldo, Alessandro},
  booktitle = "Advances in Neural Information Processing Systems",
  year={2024}
}

@inproceedings{han2024fusemoe,
  title={FuseMoE: Mixture-of-Experts Transformers for Fleximodal Fusion},
  author={Han, Xing and Nguyen, Huy and Harris, Carl and Ho, Nhat and Saria, Suchi},
  booktitle = "Advances in Neural Information Processing Systems",
  year={2024}
}

@inproceedings{nguyen2025cosine,
    author = {Huy Nguyen and Pedram Akbarian and Trang Pham and Trang Nguyen and Shujian Zhang and Nhat Ho},
    title = {Statistical Advantages of Perturbing Cosine Router in Mixture of Experts},
    booktitle = {International Conference on Learning Representations},
    year = 2025
}

@inproceedings{dai2024deepseekmoe,
    title = "{D}eep{S}eek{M}o{E}: Towards Ultimate Expert Specialization in Mixture-of-Experts Language Models",
    author = "Dai, Damai  and
      Deng, Chengqi  and
      Zhao, Chenggang  and
      Xu, R.x.  and
      Gao, Huazuo  and
      Chen, Deli  and
      Li, Jiashi  and
      Zeng, Wangding  and
      Yu, Xingkai  and
      Wu, Y.  and
      Xie, Zhenda  and
      Li, Y.k.  and
      Huang, Panpan  and
      Luo, Fuli  and
      Ruan, Chong  and
      Sui, Zhifang  and
      Liang, Wenfeng",
    booktitle = "Proceedings of the 62nd Annual Meeting of the Association for Computational Linguistics (Volume 1: Long Papers)",
    month = aug,
    year = "2024",
    publisher = "Association for Computational Linguistics",
    doi = "10.18653/v1/2024.acl-long.70",
    pages = "1280--1297",
}

@article{bing2025learninglargesoftmaxmixtures,
      title={Learning large softmax mixtures with warm start {EM}}, 
      author={Xin Bing and Florentina Bunea and Jonathan Niles-Weed and Marten Wegkamp},
      journal={arXiv preprint arXiv:2409.09903},
      year={2025}
}

@article{comanici2025gemini25pushingfrontier,
      title={Gemini 2.5: Pushing the Frontier with Advanced Reasoning, Multimodality, Long Context, and Next Generation Agentic Capabilities}, 
      author={GeminiTeam},
      journal={arXiv preprint arXiv:2507.06261},
      year={2025}
}

@article{abdin2024phi3,
  title={Phi-3 Technical Report: A Highly Capable Language Model Locally on Your Phone},
  author={Marah Abdin and others},
  journal={arXiv preprint arXiv:2404.14219},
  year={2024}
}

@article{grun2008iden,
  author={B. Grün and F. Leisch},
  title={Identifiability of Finite Mixtures of Multinomial Logit Models with Varying and Fixed Effects},
  journal={Journal of Classification},
  year={2008},
  volume={25},
  number={2},
  pages={225-247}
}

@article{nguyen2025deepseekmoe,
      title={On {D}eep{S}eek{M}o{E}: Statistical Benefits of Shared Experts and Normalized Sigmoid Gating}, 
      author={Huy Nguyen and Thong T. Doan and Quang Pham and Nghi D. Q. Bui and Nhat Ho and Alessandro Rinaldo},
      year={2025},
     journal = {arXiv preprint arXiv:2505.10860}
}

@article{dempster1977maximum,
  title={Maximum likelihood from incomplete data via the EM algorithm},
  author={Dempster, Arthur P and Laird, Nan M and Rubin, Donald B},
  journal={Journal of the royal statistical society: series B (methodological)},
  volume={39},
  number={1},
  pages={1--22},
  year={1977},
  publisher={Wiley Online Library}
}

@article{chamroukhi2009time,
  title={Time series modeling by a regression approach based on a latent process},
  author={Chamroukhi, Faicel and Sam{\'e}, Allou and Govaert, G{\'e}rard and Aknin, Patrice},
  journal={Neural Networks},
  volume={22},
  number={5-6},
  pages={593--602},
  year={2009},
  publisher={Elsevier}
}

@article{deepseekv3,
  title={Deepseek-v3 technical report},
  author={DeepSeek-AI and others},
  journal={arXiv preprint arXiv:2412.19437},
  year={2024}
}

@article{qwen2025,
  title={Qwen2.5 Technical Report},
  author={Qwen and others},
  journal={arXiv preprint arXiv:2412.15115},
  year={2025}
}

@article{deepseekv2,
  title={DeepSeek-V2: A Strong, Economical, and Efficient Mixture-of-Experts Language Model},
  author={DeepSeek-AI and others},
  journal={arXiv preprint arXiv:2405.04434},
  year={2024}
}

@ARTICLE{yuksel2012twenty,
  author={Yuksel, Seniha Esen and Wilson, Joseph N. and Gader, Paul D.},
  journal={IEEE Transactions on Neural Networks and Learning Systems}, 
  title={Twenty Years of Mixture of Experts}, 
  year={2012},
  volume={23},
  number={8},
  pages={1177-1193},
  doi={10.1109/TNNLS.2012.2200299}}

@inproceedings{kwon_em_2020,
	series = {Proceedings of {Machine} {Learning} {Research}},
	title = {{EM} {Converges} for a {Mixture} of {Many} {Linear} {Regressions}},
	volume = {108},
	url = {https://proceedings.mlr.press/v108/kwon20a.html},
	booktitle = {Proceedings of the {Twenty} {Third} {International} {Conference} on {Artificial} {Intelligence} and {Statistics}},
	publisher = {PMLR},
	author = {Kwon, Jeongyeol and Caramanis, Constantine},
	editor = {Chiappa, Silvia and Calandra, Roberto},
	month = aug,
	year = {2020},
	pages = {1727--1736},
}

@article{chen_improved_1999,
	title = {Improved learning algorithms for mixture of experts in multiclass classification},
	volume = {12},
	issn = {0893-6080},
	url = {https://www.sciencedirect.com/science/article/pii/S089360809900043X},
	doi = {https://doi.org/10.1016/S0893-6080(99)00043-X},
	number = {9},
	journal = {Neural Networks},
	author = {Chen, K. and Xu, L. and Chi, H.},
	year = {1999},
	pages = {1229--1252},
}

@inproceedings{liang_m3vit_2022,
	title = {M$^3${ViT}: {Mixture}-of-{Experts} {Vision} {Transformer} for {Efficient} {Multi}-task {Learning} with {Model}-{Accelerator} {Co}-design},
	booktitle = {{NeurIPS}},
	author = {Liang, Hanxue and Fan, Zhiwen and Sarkar, Rishov and Jiang, Ziyu and Chen, Tianlong and Zou, Kai and Cheng, Yu and Hao, Cong and Wang, Zhangyang},
	year = {2022},
}

@book{Vandegeer-2000,
author= "S. van de Geer",
title="Empirical Processes in M-estimation",
publisher= "Cambridge University Press",
year="2000"
}

@article{Jacob_Jordan-1991,
	author="R. A. Jacobs and M. I. Jordan and S. J. Nowlan and G. E. Hinton",
	title="Adaptive mixtures of local experts",
	journal="Neural Computation",
	volume="3",
	page="79-87",
	year="1991"
}

@INPROCEEDINGS{shazeer2017topk,
   AUTHOR = "N. Shazeer and A. Mirhoseini and K. Maziarz and A. Davis and Q. Le and G. Hinton and J. Dean",
   TITLE = "Outrageously Large Neural Networks: The Sparsely-Gated Mixture-of-Experts Layer",
   BOOKTITLE = 	 "In International Conference on Learning Representations", 
   YEAR = 	 2017
}

@InProceedings{manole22refined,
  title = 	 {Refined Convergence Rates for Maximum Likelihood Estimation under Finite Mixture Models},
  author =       {T. Manole and N. Ho},
  booktitle = 	 {Proceedings of the 39th International Conference on Machine Learning},
  pages = 	 {14979--15006},
  year = 	 {2022},
  volume = 	 {162},
  series = 	 {Proceedings of Machine Learning Research},
  month = 	 {17--23 Jul},
  publisher =    {PMLR}
}

@inproceedings{
diep2025on,
title={{O}n {Z}ero-{I}nitialized {A}ttention: {O}ptimal {P}rompt and {G}ating {F}actor {E}stimation},
author={Nghiem Tuong Diep and Huy Nguyen and Chau Nguyen and Minh Le and Duy Minh Ho Nguyen and Daniel Sonntag and Mathias Niepert and Nhat Ho},
booktitle={Forty-second International Conference on Machine Learning},
year={2025}
}

@inproceedings{
truong2025replora,
title={{R}ep{L}o{{R}{A}}: {R}eparameterizing {L}ow-rank {A}daptation via the {P}erspective of {M}ixture of {E}xperts},
author={Tuan Truong and Chau Nguyen and Huy Nguyen and Minh Le and Trung Le and Nhat Ho},
booktitle={Forty-second International Conference on Machine Learning},
year={2025}
}

@inproceedings{
le2025revisiting,
title={{R}evisiting {P}refix-tuning: {S}tatistical {B}enefits of {R}eparameterization among {P}rompts},
author={Minh Le and Chau Nguyen and Huy Nguyen and Quyen Tran and Trung Le and Nhat Ho},
booktitle={The Thirteenth International Conference on Learning Representations},
year={2025},
url={https://openreview.net/forum?id=QjTSaFXg25}
}

@inproceedings{
yun2024flexmoe,
title={Flex-MoE: Modeling Arbitrary Modality Combination via the Flexible Mixture-of-Experts},
author={Sukwon Yun and Inyoung Choi and Jie Peng and Yangfan Wu and Jingxuan Bao and Qiyiwen Zhang and Jiayi Xin and Qi Long and Tianlong Chen},
booktitle={The Thirty-eighth Annual Conference on Neural Information Processing Systems},
year={2024}
}

@article{grattafiori2024llama3,
  title={The Llama 3 Herd of Models},
  author={Aaron Grattafiori and Abhimanyu Dubey and Abhinav Jauhri and Abhinav Pandey and Abhishek Kadian and Ahmad Al-Dahle and Aiesha Letman and Akhil Mathur and others},
  journal={arXiv preprint arXiv:2407.21783},
  year={2024}
}

@article{jiang2024mixtral,
      title={Mixtral of Experts}, 
      author={Albert Q. Jiang and Alexandre Sablayrolles and Antoine Roux and Arthur Mensch and Blanche Savary and Chris Bamford and Devendra Singh Chaplot and Diego de las Casas and Emma Bou Hanna and Florian Bressand and Gianna Lengyel and Guillaume Bour and Guillaume Lample and Lélio Renard Lavaud and Lucile Saulnier and Marie-Anne Lachaux and Pierre Stock and Sandeep Subramanian and Sophia Yang and Szymon Antoniak and Teven Le Scao and Théophile Gervet and Thibaut Lavril and Thomas Wang and Timothée Lacroix and William El Sayed},
      year={2024},
      Journal = {arxiv preprint arxiv 2401.04088}
}

@inproceedings{lepikhin_gshard_2021,
	title = {{GS}hard: {Scaling} {Giant} {Models} with {Conditional} {Computation} and {Automatic} {Sharding}},
	booktitle = {International {Conference} on {Learning} {Representations}},
	author = {D. Lepikhin and H. Lee and Y. Xu and D. Chen and O. Firat and Y. Huang and M. Krikun and N. Shazeer and Z. Chen},
	year = {2021},
}

@inproceedings{Riquelme2021scalingvision,
 author = {C. Riquelme and J. Puigcerver and B. Mustafa and M. Neumann and R. Jenatton and A. Susano Pint and D. Keysers and N. Houlsby},
 booktitle = {Advances in Neural Information Processing Systems},
 pages = {8583--8595},
 publisher = {Curran Associates, Inc.},
 title = {Scaling Vision with Sparse Mixture of Experts},
 volume = {34},
 year = {2021}
}

@inproceedings{chen2022theory,
 author = {Chen, Zixiang and Deng, Yihe and Wu, Yue and Gu, Quanquan and Li, Yuanzhi},
 booktitle = {Advances in Neural Information Processing Systems},
 editor = {S. Koyejo and S. Mohamed and A. Agarwal and D. Belgrave and K. Cho and A. Oh},
 pages = {23049--23062},
 publisher = {Curran Associates, Inc.},
 title = {Towards Understanding the Mixture-of-Experts Layer in Deep Learning},
 volume = {35},
 year = {2022}
}

@article{pham2022functional,
  title     = {Functional Mixture-of-Experts for Classification},
  author    = {Pham, Nhat Thien and Chamroukhi, Fa{\"i}cel},
  journal   = {arXiv preprint arXiv:2202.13934},
  year      = {2022}
}

@inproceedings{wang2025expressive,
  title     = {On the Expressive Power of Mixture-of-Experts for Structured Complex Tasks},
  author    = {Wang, Mingze and E, Weinan},
  booktitle = {Advances in Neural Information Processing Systems 38 (NeurIPS)},
  year      = {2025},
  note      = {Spotlight}
}

@article{rigollet2025granularity,
  title   = {The Power of Fine-Grained Experts: Granularity Boosts Expressivity in Mixture of Experts},
  author  = {Boix-Adsera, Enric and Rigollet, Philippe},
  journal = {arXiv preprint arXiv:2505.06839},
  year    = {2025}
}

@InProceedings{nguyen2024general,
  title = 	 {A General Theory for Softmax Gating Multinomial Logistic Mixture of Experts},
  author =       {Nguyen, Huy and Akbarian, Pedram and Nguyen, Trungtin and Ho, Nhat},
  booktitle = 	 {Proceedings of the 41st International Conference on Machine Learning},
  pages = 	 {37617--37648},
  year = 	 {2024}
}

@article{faria2010regression,
author = {Susana Faria and Gilda Soromenho},
title = {Fitting mixtures of linear regressions},
journal = {Journal of Statistical Computation and Simulation},
volume = {80},
number = {2},
pages = {201-225},
year = {2010},
publisher = {Taylor & Francis}
}

@inproceedings{
li2023sparse,
title={Sparse Mixture-of-Experts are Domain Generalizable Learners},
author={Bo Li and Yifei Shen and Jingkang Yang and Yezhen Wang and Jiawei Ren and Tong Che and Jun Zhang and Ziwei Liu},
booktitle={The Eleventh International Conference on Learning Representations },
year={2023}
}

@inproceedings{csordas2023approximating,
    title = "Approximating Two-Layer Feedforward Networks for Efficient Transformers",
    author = {Csord{\'a}s, R{\'o}bert  and
      Irie, Kazuki  and
      Schmidhuber, J{\"u}rgen},
    editor = "Bouamor, Houda  and
      Pino, Juan  and
      Bali, Kalika",
    booktitle = "Findings of the Association for Computational Linguistics: EMNLP 2023",
    month = dec,
    year = "2023",
    address = "Singapore",
    publisher = "Association for Computational Linguistics",
    pages = "674--692"
}
\bibliographystyle{abbrv}
\end{document}